\documentclass{article}

\usepackage[preprint]{neurips_2026}

\usepackage[utf8]{inputenc} % allow utf-8 input
\usepackage[T1]{fontenc}    % use 8-bit T1 fonts
\usepackage{hyperref}       % hyperlinks
\hypersetup{hypertexnames=false}
\usepackage{url}            % simple URL typesetting
\usepackage{booktabs}       % professional-quality tables
\usepackage{amsfonts}       % blackboard math symbols
\usepackage{nicefrac}       % compact symbols for 1/2, etc.
\usepackage{microtype}      % microtypography
\usepackage{xcolor}         % colors
\usepackage{graphicx}       % figures

\usepackage{amsmath,amssymb,amsthm}
\usepackage{bm}
\usepackage{nicefrac}
\usepackage{mathtools}
\usepackage{microtype}
\usepackage{algorithm}
\usepackage[noend]{algpseudocode}

\newcommand{\R}{\mathbb{R}}
\newcommand{\E}{\mathbb{E}}
\newcommand{\ones}{\mathbf{1}}
\newcommand{\cF}{\mathcal{F}}
\DeclareMathOperator*{\argmax}{arg\,max}

\usepackage{hyperref}       % hyperlinks
\hypersetup{
	colorlinks   = true, %Colours links instead of ugly boxes
	urlcolor     = blue, %Colour for external hyperlinks
	linkcolor    = blue, %Colour of internal links
	citecolor   = blue %Colour of citations
}

\newtheorem{theorem}{Theorem}
\newtheorem{proposition}[theorem]{Proposition}
\newtheorem{lemma}[theorem]{Lemma}
\newtheorem{corollary}[theorem]{Corollary}
\theoremstyle{definition}

\theoremstyle{remark}
\newtheorem{remark}[theorem]{Remark}
\usepackage{amsthm}
\newtheorem{definition}{Definition}

\title{Are Stochastic Multi-objective Bandits Harder than Single-objective Bandits?}

\author{%
  Changkun Guan\\
  H. Milton Stewart School of Industrial and Systems Engineering \\
  Georgia Institute of Technology \\
  \texttt{cguan62@gatech.edu} \\
  \And
  Mengfan Xu\thanks{Corresponding author}\\
  Mechanical and Industrial Engineering \\
  University of Massachusetts Amherst \\
  \texttt{mengfanxu@umass.edu} \\
}

\begin{document}

\maketitle

%The learner does not need to recover the entire Pareto front, because committing (namely certifying) one Pareto-optimal arm already makes Pareto regret vanish.
\begin{abstract}
Multi-objective bandits have attracted increasing attention for their broad
applicability, with \(d\)-dimensional reward vectors inducing Pareto regret.
There has been a subtle debate over whether this added structure makes the
problem fundamentally harder than single-objective bandits. We answer this by showing that, in terms of Pareto regret, it is surprisingly no harder: Pareto regret scales
inversely with \(g^\dagger\), the largest objective-wise suboptimality gap, and
thus matches the smallest objective-wise classical regret. We formalize this
idea via a novel method with upper and lower confidence-bound estimators
for every arm-objective pair. It uses top-two races to compare arms within
each objective and an uncertainty-greedy rule to allocate exploration toward
the largest objective-wise gap \(g^\dagger\), until the corresponding
Pareto-optimal arm is committed to. We prove that it achieves Pareto regret of
\(O(\nicefrac{\log T}{g^\dagger})\), where \(T\) is the horizon, with
\emph{no dependence on \(d\)}. A matching lower bound of
\(\Omega(\nicefrac{\log T}{g^\dagger})\) implies optimality. We evaluate the method on synthetic and real-world datasets, confirming the
theory and achieving order-of-magnitude reductions in Pareto regret over
baselines. Real-world results further show that our method commits to a Pareto optimal arm, possibly at the cost of empirical fairness, suggesting a potential hardness absent in single-objective bandits.\end{abstract}

%Real-world results further show that our method always commits to a Pareto-optimal arm and suggest that fairness is an empirical plausible hardness absent in single-objective bandits. 

%Real-world results further show that our method commits to a Pareto-optimal arm,
%possibly at the cost of empirical fairness, suggesting a hardness absent in single-objective bandits.
%Consequently, the single-objective-like scale cannot be uniformly improved over the full stochastic MO-MAB model.
%can substantially lower Pareto regret while 

\section{Introduction}
\label{sec:introduction}
Standard multi-armed bandits model sequential decision-making under uncertainty through a single scalar reward \citep{auer2002finite, auer2002nonstochastic}. This abstraction has been highly successful across a wide range of applications, from clinical trials to recommendation systems. In many modern settings, however, actions are evaluated along several criteria simultaneously rather than by a single score.  This motivates the study of multi-objective multi-armed bandits (MO-MABs) \citep{drugan2013designing, drugan2014pareto, xu2023pareto, huyuk2021multi, busa2017multi}. In a MO-MAB problem, the learner chooses an arm at each round and observes a $d$-dimensional reward vector rather than a scalar. Because the objectives may compete, there need not be a single universally best arm. A standard benchmark is Pareto regret, which measures performance relative to the Pareto frontier of the true reward vectors.

A natural intuition in stochastic MO-MABs is that the problem should become
harder as the number of objectives grows \citep{drugan2013designing}. Pareto
optimality is defined through coordinate-wise comparisons across all \(d\)
dimensions, and the Pareto frontier can become increasingly intricate.
Furthermore, existing gap-dependent bounds for Pareto UCB1
\citep{drugan2013designing} depend on the collection of Pareto gaps over all
dominated arms, rather than on a single objective-wise gap. From the analysis, one might expect Pareto regret to carry
an explicit dependence on the ambient multi-objective complexity and to increase
with dimensionality, \emph{suggesting that stochastic multi-objective bandits might be intrinsically
harder.}

Recent work by \citep{xu2023pareto} shows that, in adversarial
multi-objective bandits, Pareto regret can scale at the same order as
single-objective bandit regret and can be linked to minimal one-dimensional
marginal regret. This suggests that \emph{Pareto regret need not inherit the
full apparent complexity of the ambient multi-objective structure; from this
Pareto-regret perspective, adversarial multi-objective bandits are not
necessarily harder.} However, this leads to a subtle debate rather than a settled conclusion, because
these insights do not directly resolve the stochastic setting studied in
\citep{drugan2013designing}. First, the
adversarial result concerns minimax worst-case rates over the horizon and
adopts a regret notion different from that in stochastic multi-objective
bandits. Second, although the connection between Pareto regret and marginal
regret is conceptually suggestive, existing methods may still scale with the
largest marginal-regret scale across dimensions in theory. It remains
unexplored, under i.i.d.\ rewards in the stochastic full-feedback model, how to
determine the gap-dependent scale of Pareto regret, how it relates to classical
single-objective regret, and how to design an algorithm that achieves this
scale. This scale is crucial: it determines whether Pareto regret reflects the
hardest objective-wise gap or the easiest gap, and hence \emph{whether
multi-objective bandits are genuinely harder than single-objective bandits.}
Section~\ref{sec:related} presents a thorough review of related work.%\textcolor{blue}{MX to do: scale of marginal regret is important but not resolved by this paper as well.}

%While promising, this presents a debate. However, this debate itself deservedoes not settle the present subtle debate.

The main research problem is therefore 

\begin{center}
\emph{How does Pareto regret relate to marginal regret in stochastic bandits, and can this relationship be exploited to achieve the easiest marginal-regret scale, avoiding the full cost by \citep{drugan2013designing} and \citep{xu2023pareto}?}\end{center}

Our affirmative answer is that \textbf{stochastic MO-MAB is no harder than the
single-objective bandit problem in the Pareto regret sense}. Specifically, in
the stochastic full-feedback setting, Pareto regret \citep{drugan2013designing}
can again be controlled through marginal regret, consistent with the conclusion
from the adversarial setting \citep{xu2023pareto}, despite the different regret
definition. To this end, we identify a problem-class property under which every
marginal maximizer is Pareto-optimal. The relevant instance-dependent quantity
is the largest objective-wise suboptimality gap, defined formally in
Section~\ref{subsec:objective_gap}. This is encouraging: it shows that one can
conceptually achieve the easiest marginal-regret scale, echoing the insight of
\citep{xu2023pareto}. %through one arm that maximizes a single objective, rather than through the full Pareto set.
%\textcolor{blue}{MX: certification is not defined} \textcolor{orange}{CG: Already fixed all the certification words before Section 3.}

%we propose a novel method that searches for the best arm that actually achieves the minim.

%\textcolor{blue}{MX: there should still be a paragraph here on the method, e.g., what is the high level idea of the method, and also highlight the novelty.}

%\textcolor{orange}{CG: I kind of writing a paragraph in the remark 2 that might highlight the difference between our methods with the other methods. But I fully agree that we need a paragraph here as well.}

%It is enough to find one arm in \(A^\star\), because all later pulls of that arm add zero Pareto regret. If every arm that maximizes an objective is Pareto-optimal, then identifying an arm as the unique maximizer of some objective is enough to find an arm in \(A^\star\).

We propose a novel algorithm that exploits this structural
relationship. The key observation is that minimizing Pareto regret does not
require learning the entire Pareto set; instead, it suffices to identify an
objective with the largest objective-wise suboptimality gap (but \emph{unknown}) and commit to its
best arm. The algorithm performs an adaptive search
over objective-wise two-arm comparisons. For each objective, the algorithm maintains the two arms with the largest upper
confidence bounds, encouraging additional exploration under objective-wise
uncertainty. It then uses the gap between lower and upper confidence estimates
to infer the objective-wise suboptimality gap and checks whether the leading arm
is separated from the runner-up by their confidence intervals. Before such a
separation occurs, the next pull is chosen from the two arms along the objective with the largest confidence-radius sum. This serves as a uniform
upper-bound proxy for the objective-wise gap; within that comparison, the arm
with the larger confidence radius is sampled. This creates a two-level exploration--exploitation mechanism: across
objectives, the algorithm explores the comparison with the largest uncertainty
and potential gap, while exploiting the proxy for the
objective-wise gap; within that objective, it samples the more uncertain arm.
Once a separation occurs, the confidence radius has fallen below the largest
certified objective-wise gap. The algorithm then commits to the corresponding
objective and selects its best arm. %The novelty is that the algorithm neither fixes a scalar objective in advance nor estimates the full Pareto set. Instead, it performs a meta-level exploration--exploitation procedure over both objectives and arms. %I%t uses all objectives as possible ways to identify one Pareto-optimal arm, which allows the regret bound to depend on the largest objective-wise suboptimality gap rather than on the full collection of dominated-arm Pareto gaps.}

Additionally, we prove a finite-time Pareto regret bound of order $O\!(K \nicefrac{\log T}{g^\dagger}),$ where \(g^\dagger\) denotes the largest objective-wise suboptimality gap. The leading term has no explicit multiplicative dependence on the number of objectives and depends on this objective-wise suboptimality gap rather than on a sum over dominated-arm Pareto gaps. We also prove an asymptotic lower bound of order $\Omega\!(K \nicefrac{\log T}{g^\dagger})$ on a natural duplicated-coordinate Bernoulli subclass of the full model. Consequently, the order \(K\nicefrac{\log T}{g^\dagger}\) cannot be improved in general by any uniformly good policy, and our upper bound is tight up to constant factors. As a byproduct, it yields a high-probability guarantee that the committed arm is also Pareto-optimal.

%To complement the theory, we present a numerical study designed to examine whether finite-horizon Pareto regret is governed more by the single certification gap $g^\dagger$ than by the full collection of dominated-arm Pareto gaps. We consider a synthetic family in which $g^\dagger$ is kept fixed while both the Pareto gap and the number of dominated arms are varied. This isolates the setting in which Pareto UCB1 becomes costly because its leading term depends on many small arm-wise Pareto gaps, whereas our method only needs to certify one objective winner. Across the tested configurations, the width-guided first-certification policy attains lower finite-horizon Pareto regret than Pareto UCB1. The trajectory plots further show how this advantage develops once certification is achieved.
To examine the practical effectiveness of the proposed method, we evaluate the width-guided method on real-world 
Spotify and KuaiRec benchmarks and on a controlled synthetic family. The results
confirm the theory across both types of datasets and further demonstrate the
method's practical value against widely used baselines. On the real-data
subsets, Width-guided attains the lowest Pareto regret in the reported groups,
with order-of-magnitude ($\geq \mathbf{111}$ times)  improvements on Spotify datasets, and always selects a Pareto-optimal arm
in every reported Spotify and KuaiRec group. We also report an empirical
fairness measure. Although Width-guided is designed only to reduce Pareto regret
and identify a Pareto-optimal arm, its fairness behavior is not uniformly worse:
it loses fairness in some groups, but is surprisingly best-of-all-worlds in
Pareto regret, Pareto-optimal-arm detection, and empirical fairness in others.
This mixed behavior suggests that fairness deserves theory-oriented study, beyond as an empirical measure.

% Since synethtic study not in the main text. I remove it: The synthetic study is designed to test whether Pareto regret is controlled by \(g^\dagger\) rather than by all dominated-arm Pareto gaps. It keeps \(g^\dagger\) fixed and varies either  (\Delta_{\min}^{\mathrm P}\) or the number of positive-Pareto-gap arms close to the Pareto front. Thus the scale in our bound remains fixed, while the gap-dependent bound for Pareto UCB1 becomes larger.}

%\textcolor{blue}{MX: perhaps revise this last paragraph to match the insights from newly added experiments? }

%\textcolor{blue}{CG: Fixed}
%The remainder of the paper is organized as follows. Section~\ref{sec:related} reviews related work. Section~\ref{sec:problem} introduces the stochastic MO-MAB model, Pareto regret, and the certification gap that drives our analysis. Section~\ref{sec:algorithm} presents the proposed width-guided first-certification algorithm. Section~\ref{sec:results} establishes our main guarantees, including a finite-time upper bound, an asymptotic lower bound, and a comparison with Pareto UCB1. Section~\ref{sec:computational} reports the numerical study, and Section~\ref{sec:conclusion} concludes. Omitted proofs and supplementary derivations are collected in Appendix

\section{Problem Formulation}
\label{sec:problem}

We formulate the stochastic multi-objective bandit problem and introduce the
notation used throughout the paper. We consider a stochastic multi-objective bandit problem with \(K\) arms and
\(d\) objectives. For each arm \(a\in[K]:=\{1,\dots,K\}\), let
\(Y_{a,1},Y_{a,2},\ldots\in[0,1]^d\) be i.i.d. reward vectors with mean
\(\mu_a=(\mu_a^{(1)},\ldots,\mu_a^{(d)})\in\mathbb{R}^d\). The reward sequences
are independent across arms. At each round \(1 \leq t \leq T\) where $T$ is time horizon, a policy (decision maker) selects an arm
\(A_t\in[K]\) based on the history observed so far and receives the next unused
sample from that arm, \(X_t:=Y_{A_t,N_{A_t}(t-1)+1}\), where
\(N_a(t):=\sum_{s=1}^t \mathbf{1}\{A_s=a\}\) denotes the number of pulls of arm
\(a\) up to time \(t\). A key feature of the feedback model is that pulling an
arm reveals the entire \(d\)-dimensional reward vector of that arm: all
objectives of the selected arm are observed simultaneously, while no information
is received about unchosen arms.

\subsection{Pareto Order, Pareto Optimality, and Pareto Regret}

\paragraph{Pareto order.}
We compare vector-valued rewards using the Pareto order
\citep{drugan2013designing}. For \(x,y\in\mathbb{R}^d\), write
\(x\succeq y\) if \(x^{(j)}\ge y^{(j)}\) for all \(j\in[d]\), and write
\(x\succ y\) if \(x\succeq y\) and \(x\neq y\). Thus, \(x\succ y\) means that
\(x\) is weakly larger than \(y\) in every coordinate and strictly larger in at
least one coordinate. If neither \(x\succeq y\) nor \(y\succeq x\) holds, then
\(x\) and \(y\) are incomparable.

\paragraph{Pareto optimality.}
An arm \(a\) dominates arm \(b\) if \(\mu_a\succ\mu_b\). The Pareto set is
\(A^\star:=\{a\in[K]: \nexists b\in[K]\text{ such that }\mu_b\succ\mu_a\}\),
and its elements are called Pareto-optimal arms.

Following \citep{drugan2013designing}, the Pareto suboptimality gap of arm
\(a\) is
\[
\Delta_a^{\mathrm P}
:=
\inf\left\{
\epsilon\ge 0:
\mu_a+\epsilon\mathbf{1}
\text{ is not dominated by any } \mu_b,\ b\in[K]
\right\}.
\]
Every Pareto-optimal arm has zero Pareto suboptimality gap, but the converse need not hold. For example, if \(\mu_1=(1,1)\) and \(\mu_2=(1,0)\), then arm \(2\) is
dominated by arm \(1\), while \(\Delta_2^{\mathrm P}=0\), since
\((1+\epsilon,\epsilon)\) is not dominated by \((1,1)\) for any
\(\epsilon>0\). Thus, zero Pareto gap alone does not necessarily imply Pareto
optimality.

\paragraph{Pareto regret.}
We measure performance by Pareto regret \citep{drugan2013designing} over horizon \(T\):
\[
R_T^{\mathrm P}
:=
\sum_{t=1}^T \Delta_{A_t}^{\mathrm P}
=
\sum_{a:\Delta_a^{\mathrm P}>0}
\Delta_a^{\mathrm P}N_a(T).
\]

\emph{Our analysis relates Pareto regret to the objective-wise gaps defined below.}

\subsection{Pareto Regret and Objective-Wise Gaps}
\label{subsec:objective_gap}

We first define the objective-wise gaps used in the regret bound. For objective \(j\), define \(\mu_{(1),j}:=\max_{b\in[K]}\mu_b^{(j)}\) and
\(\Delta_a^{(j)}:=\mu_{(1),j}-\mu_a^{(j)}\). Thus,
\(\Delta_a^{(j)}=0\) means that arm \(a\) is a marginal maximizer of objective
\(j\). Any such arm has zero Pareto gap: for every \(\epsilon>0\),
\(\mu_a^{(j)}+\epsilon>\max_{b\in[K]}\mu_b^{(j)}\), so
\(\mu_a+\epsilon\mathbf{1}\) cannot be dominated by any arm.

For
objective \(j\), let
\(M_j:=\arg\max_{a\in[K]}\mu_a^{(j)}
=\{a\in[K]:\Delta_a^{(j)}=0\}\)
be the set of marginal maximizers. If \(M_j\) is a singleton, write
\(M_j=\{a_j^\star\}\) and define
\[
g_j
:=
\mu_{a_j^\star}^{(j)}
-
\max_{a\neq a_j^\star}\mu_a^{(j)}.
\]
If \(M_j\) contains multiple arms, set \(g_j:=0\). Thus, \(g_j>0\) exactly
when objective \(j\) has a unique marginal maximizer whose objective-\(j\) mean is strictly larger than every other arm's objective-\(j\) mean. Finally, define \(j^\dagger\in\arg\max_{j\in[d]}g_j\) and \(g^\dagger:=g_{j^\dagger}\). The parameter \(g^\dagger\) is the largest objective-wise suboptimality gap. For objective \(j\), define the marginal regret
\(
R_T^{(j)}
:=
\sum_{t=1}^T(\mu_{(1),j}-\mu_{A_t}^{(j)})
=
\sum_{a=1}^K \Delta_a^{(j)}N_a(T).
\)

\paragraph{Debate and Objective.} Classical gap-dependent bounds for marginal regret scale as \(1/g_j\)
\citep{lai}. \emph{The ongoing debate is whether Pareto regret is governed by the
largest marginal-regret scale across objectives, \(1/\min_j g_j\), together
with an explicit dependence on the dimension \(d\), as in
\citep{drugan2013designing}, or instead by the smallest such scale,
\(1/\max_j g_j = 1/g^\dagger\).} This debate is important: the former
suggests that multi-objective bandits are intrinsically harder than
single-objective bandits, whereas the latter suggests that they are not. For
example, suppose \(d=2\), \(g_1=10^{-3}\), and \(g_2=1\). Then
\(g^\dagger=1\). The classical marginal regret scale is \(1/g_1=10^3\) for
objective \(1\), but only \(1/g_2=1\) for objective \(2\). A bound governed by
the hardest objective would scale with \(1/\min_j g_j=10^3\), possibly with an
additional factor depending on \(d\), whereas the bound in the latter case scales with
\(1/g^\dagger=1\). Thus, Pareto regret follows the easier objective-wise scale
rather than the hardest one. We resolve this debate herein.

%\textcolor{blue}{MX: would it be possible to add a proposition saying that Zero-Pareto-gap does not guarantee Pareto optimal?}

Motivated by the definitions of Pareto regret and marginal regret, we identify
a subtle distinction between zero Pareto gap and Pareto optimality: the two
notions are not equivalent. Prop.~\ref{prop:zero_gap_not_pareto} shows
that, although the implication
\(a\in A^\star \Rightarrow \Delta_a^{\mathrm P}=0\) always holds, the reverse
implication can fail. Consequently, marginal regret does not automatically
correspond to Pareto optimality and may be less meaningful in the
multi-objective setting without additional structure. To clarify when
objective-wise optimality provides a valid Pareto certificate, we consider a
consistent problem class in Def.~\ref{ass:marginal_pareto_consistency}.

\begin{proposition}[Zero Pareto gap $\neq$ Pareto optimality]
\label{prop:zero_gap_not_pareto}
There exists an arm \(a\notin A^\star\) with \(\Delta_a^{\mathrm P}=0\).
\end{proposition}

\begin{definition}[Marginal-Pareto consistency]
\label{ass:marginal_pareto_consistency}
Every marginal maximizer is Pareto-optimal:
\[
\Delta_a^{(j)}=0
\quad\Longrightarrow\quad
a\in A^\star,
\qquad
\forall a\in[K],\ j\in[d].
\]
Equivalently, \(\bigcup_{j=1}^d
\arg\max_{a\in[K]}\mu_a^{(j)}\subseteq A^\star\).
\end{definition}

Def.~\ref{ass:marginal_pareto_consistency} identifies the problem class
in which objective-wise optimality certifies a Pareto-optimal arm, thereby
making marginal regret meaningful for Pareto optimality. It is weaker than
requiring every zero-Pareto-gap arm to be Pareto-optimal, since it only
restricts zero-gap arms that arise as marginal maximizers. It is also weaker
than requiring every objective to have a unique maximizer: uniqueness implies
the condition, but the condition allows ties as long as all tied marginal
maximizers are Pareto-optimal. If \(g^\dagger>0\), then objective
\(j^\dagger\) has a unique marginal maximizer \(a_{j^\dagger}^\star\), and
under Def.~\ref{ass:marginal_pareto_consistency}, this arm belongs to
\(A^\star\).

\section{Methodology: Width-Guided Objective-Wise UCB}
\label{sec:algorithm}

We propose a method that uses confidence bounds to verify that an arm has a strictly larger mean than every other arm on one objective. Under Def.~\ref{ass:marginal_pareto_consistency}, such an arm belongs to \(A^\star\), and therefore has zero Pareto suboptimality gap. Throughout the rest of the paper, a \emph{certificate} means the confidence-bound test that verifies this statement. Once the certificate is accepted, the policy \emph{commits} to the certified arm and pulls this arm for all remaining rounds.

Before a certificate is accepted, the policy must decide which arm to sample. Algorithm \ref{alg:ucb_lcb_guided} presents the complete pseudo code.  Algorithm~\ref{alg:ucb_lcb_guided} compares arms separately within each objective using upper and lower confidence bounds. For each objective, it keeps
the arms with the largest and second-largest upper confidence bounds. It then chooses the objective whose two retained arms have the largest sum of confidence radii, and pulls the arm among those two with the larger confidence radius. Since one pull reveals the full reward vector of the selected arm, this pull updates the estimates of all objectives for that arm. Thus the algorithm does not need to estimate the whole Pareto set. It stops after one objective-wise comparison verifies a unique maximizer (with the largest objective-wise sub-optimality gap), at which point all future pulls have
zero Pareto regret. We now define the empirical quantities used to implement these comparisons.

After the warm start, let
\(
\widehat{\mu}_a^{(j)}(t-1)
:=
\nicefrac{1}{N_a(t-1)}
\sum_{s=1}^{t-1} X_s^{(j)} \mathbf{1}\{A_s=a\}
\)
denote the empirical mean of objective $j$ for arm $a$ based on its first $N_a(t-1)$ samples.

At round $t>K$, define the fixed-horizon confidence radius \(
\beta_a(t) := \sqrt{\nicefrac{2\log T}{N_a(t-1)}}.
\)
%We use a fixed-horizon radius to keep the finite-time analysis transparent and to match the radius used in Theorem~\ref{thm:regret}. 
For every arm $a$ and objective $j$,
\( 
U_a^{(j)}(t) := \widehat{\mu}_a^{(j)}(t-1) + \beta_a(t),
L_a^{(j)}(t) := \widehat{\mu}_a^{(j)}(t-1) - \beta_a(t).
\)
For each objective $j$, let
\( 
b_j(t) := \argmax_{a \in [K]} U_a^{(j)}(t),
c_j(t) := \argmax_{a \in [K]\setminus\{b_j(t)\}} U_a^{(j)}(t),
\)
and define the pair width
\(
W_j(t) := \beta_{b_j(t)}(t) + \beta_{c_j(t)}(t).
\) If some objective $j$ already satisfies
\(
L_{b_j(t)}^{(j)}(t) > U_{c_j(t)}^{(j)}(t),
\)
then \(b_j(t)\) is certified as a Pareto-optimal arm and $j$ is the objective with the largest sub-optimality gap, and the algorithm stores
\(b_j(t)\). If no objective satisfies this inequality, the algorithm chooses an
objective with largest \(W_j(t)\) and samples the arm in
\(\{b_j(t),c_j(t)\}\) with the larger confidence radius $\beta(t)$.

%When several empirical scores are exactly equal, ties are broken uniformly at random. This algorithmic tie-breaking only concerns empirical scores; population ties are already reflected by $g_j=0$ in the objective-wise gap notation above.

\begin{algorithm}[t]
\caption{Width-Guided First-Certification UCB}
\label{alg:ucb_lcb_guided}
\small
\begin{algorithmic}[1]
\Require Number of arms $K$, number of objectives $d$, horizon $T \ge K$
\State Initialize $N_a \gets 0$, $\widehat{\mu}_a \gets 0 \in \R^d$, and stored certified arm $\widehat a \gets \emptyset$
\Statex \textit{Warm start: pull each arm once}
\For{$a=1,2,\dots,K$}
    \State Pull arm $a$ and observe $X_a^{\mathrm{init}}$
    \State $N_a \gets 1$, \ $\widehat{\mu}_a \gets X_a^{\mathrm{init}}$
\EndFor
\For{$t=K+1,K+2,\dots,T$}
    \State $\beta_a(t) \gets \sqrt{2\log T / N_a}$ for all $a \in [K]$
    \State $U_a^{(j)}(t) \gets \widehat{\mu}_a^{(j)} + \beta_a(t)$ and $L_a^{(j)}(t) \gets \widehat{\mu}_a^{(j)} - \beta_a(t)$ for all $a,j$
    \For{$j=1,2,\dots,d$}
        \State $b_j(t) \gets \argmax_{a \in [K]} (\widehat{\mu}_a^{(j)} + \beta_a(t))$
        \State $c_j(t) \gets \argmax_{a \in [K]\setminus\{b_j(t)\}} (\widehat{\mu}_a^{(j)} + \beta_a(t))$
        \State $W_j(t) \gets \beta_{b_j(t)}(t) + \beta_{c_j(t)}(t)$
    \EndFor
    \If{$\widehat a \neq \emptyset$}
        \State $A_t \gets \widehat a$
    \ElsIf{some objective $j$ satisfies $L_{b_j(t)}^{(j)}(t) > U_{c_j(t)}^{(j)}(t)$}
        \State choose one such objective $\widehat j$, store $\widehat a \gets b_{\widehat j}(t)$, and set $A_t \gets \widehat a$
    \Else
        \State $j_t \gets \arg\max_{j \in [d]} W_j(t)$
        \State $A_t \gets \arg\max_{a \in \{b_{j_t}(t),c_{j_t}(t)\}} \beta_a(t)$
    \EndIf
    \State Pull arm $A_t$ and observe $X_t$
    \State $N_{A_t} \gets N_{A_t}+1$
    \State $\widehat{\mu}_{A_t} \gets \widehat{\mu}_{A_t} + \nicefrac{(X_t-\widehat{\mu}_{A_t})}{N_{A_t}}$
\EndFor
\end{algorithmic}
\end{algorithm}

For the analysis, call a round \(t>K\) \emph{non-certifying} if, at the decision point of round \(t\), no certified arm has been stored and no objective \(j\in[d]\) satisfies
\(
L_{b_j(t)}^{(j)}(t) > U_{c_j(t)}^{(j)}(t).
\)
At a non-certifying round, the algorithm selects
\(
j_t \in \arg\max_{j \in [d]} W_j(t),
b_t := b_{j_t}(t),
c_t := c_{j_t}(t),
\)
and then pulls
\(
A_t \in \arg\max_{a \in \{b_t,c_t\}} \beta_a(t).
\)

\begin{remark}[Comparison with existing algorithms]
Existing methods use confidence bounds to establish different
types of guarantees. Pareto-UCB-style algorithms build an optimistic Pareto set
and bound how often arms with positive Pareto suboptimality gap remain
optimistically nondominated. Pareto-front identification methods decide the
Pareto status of many arms. Scalar best-arm identification methods, including
LUCB-style algorithms, reduce to methods with one fixed scalar objective
\citep{kalyanakrishnan2012pac,jamieson2014best}. Algorithm~\ref{alg:ucb_lcb_guided} aims to verify that one arm is the unique
maximizer of a target objective via a two-layer exploration-exploitation
trade-off across arms and objectives, where each layer involves a new design.
It compares two arms within one objective, but the objective is chosen
adaptively through the current width \(W_j(t)\). It does not maintain an
estimated Pareto set or decide the Pareto status of all arms. It stops once one
objective-wise comparison verifies a unique maximizer with the largest
suboptimality gap. This is enough for Pareto regret as such an arm has zero Pareto suboptimality gap. It also ensures that
Pareto regret scales inversely with the largest suboptimality gap, since the
exploration sample complexity leading to the dominant regret term depends on
this gap.
\end{remark}

\section{Theoretical Analyses}
\label{sec:results}

The finite-time analysis below shows that the leading gap-dependent Pareto regret bound is controlled by \(g^\dagger\), the largest objective-wise suboptimality gap, rather than by a sum over all dominated-arm Pareto suboptimality gaps. The leading term has no explicit multiplicative factor in the number of objectives. The reason is that the algorithm stops exploration once the confidence bounds verify a unique maximizer of one objective, which is Pareto-optimal under Def.~\ref{ass:marginal_pareto_consistency}. Together with the lower bound below, this shows that the standard Pareto regret can have the same leading order as a single-objective bandit. We refer to Appendix \ref{app:proofs} for the full proofs of all theoretical results herein.

\subsection{Lower Bound Analysis}

Before stating the finite-time upper bound, we record a lower bound through a simple duplicated-coordinate subclass. This subclass is still $d$-dimensional and is part of the full stochastic MO-MAB model. The lower bound has a uniform full-class role: any algorithm with a uniformly smaller leading order over the full stochastic MO-MAB model would also have to improve on this subclass. This is the lower-bound notion matched to the title question in the Pareto regret sense. It does not claim that every stochastic MO-MAB instance has instance complexity \(K\log T/g^\dagger\), nor that \(g^\dagger\) is the only relevant geometric quantity for all possible Pareto-optimal-arm certificates. Rather, the relevant obstruction is already present inside the model. In particular, the statement is not that every geometry is governed by \(g^\dagger\), but that the full model already contains instances with the matching leading order
\[
\Omega\!(\frac{K\log T}{g^\dagger})
\lesssim
\E[R_T^{\mathrm P}]
\lesssim
O\!(\frac{K\log T}{g^\dagger}).
\]
Thus even within the multi-objective setting studied here, the unavoidable benchmark can remain of the same order as in a single-objective bandit.

\begin{proposition}[Asymptotic lower bound]
\label{prop:lower_bound}
Assume $K\ge 2$ and fix $\Delta_{\mathrm{sc}} \in (0,1/4]$. Consider the duplicated-coordinate Bernoulli instance
\(
Y_{a,n}=(Z_{a,n},\ldots,Z_{a,n})\in[0,1]^d,
\)
where
\(
Z_{1,n}\sim \mathrm{Bernoulli}\!(\frac12+\Delta_{\mathrm{sc}}),
\qquad
Z_{a,n}\sim \mathrm{Bernoulli}\!(\frac12)
\quad\text{for } a=2,\ldots,K,
\)
and all latent samples are independent across arms and pulls. Suppose moreover that, on this duplicated-coordinate subclass, the induced scalar policy obtained by identifying each observed vector $(Z_{a,n},\ldots,Z_{a,n})$ with its underlying Bernoulli sample $Z_{a,n}$ is uniformly good in the Lai--Robbins sense: for every Bernoulli mean vector $\theta\in(0,1)^K$ and every $\alpha>0$, its expected single-objective regret under $\theta$ is $o(T^\alpha)$ as $T\to\infty$. Then
\(
\liminf_{T\to\infty}\frac{\E[R_T^{\mathrm P}]}{\log T}
\ge
\frac{3}{8}\frac{K-1}{\Delta_{\mathrm{sc}}}.
\)
In particular, because $g^\dagger=\Delta_{\mathrm{sc}}$ on this instance,
\(
\E[R_T^{\mathrm P}]
=
\Omega\!(\frac{K\log T}{g^\dagger})
\text{as } T\to\infty.
\)
\end{proposition}

Proposition~\ref{prop:lower_bound} shows that multi-objective bandits contain instances with lower bound \(\Omega(K\nicefrac{\log T}{g^\dagger})\).  Therefore, no uniformly good policy can improve the order \(K\nicefrac{\log T}{g^\dagger}\) over the full model class. Together with Theorem~\ref{thm:regret}, this gives matching upper and lower bounds up to constant factors whenever \(g^\dagger>0\). The proof is deferred to Appendix~\ref{app:proofs}.

\subsection{Upper Bound Analysis}

\begin{theorem}[Finite-time bound for objective-wise Pareto-optimal certification]
\label{thm:regret}
Under the stochastic model above, suppose $K\ge 2$, $T\ge K$, rewards are bounded in $[0,1]$, and the instance has positive first-certificate scale $g^\dagger>0$. Let
\(
\Delta_{\max}^{\mathrm P} := \max_{a \in [K]} \Delta_a^{\mathrm P}.
\)
Then Algorithm~\ref{alg:ucb_lcb_guided} satisfies
\(
\E\!\left[R_T^{\mathrm P}\right]
\le
K\Delta_{\max}^{\mathrm P}
+
64\,\frac{K\log T}{g^\dagger}
+
2Kd\,\Delta_{\max}^{\mathrm P}T^{-2}.
\)
In particular,
\(
\E\!\left[R_T^{\mathrm P}\right]
=
O\!(\frac{K\log T}{g^\dagger}).
\)
\end{theorem}

The theorem is a guarantee for the objective-wise Pareto-optimal certificate family, not a characterization of all Pareto geometries. When \(g^\dagger>0\), this route supplies a statistically separated certificate and Algorithm~\ref{alg:ucb_lcb_guided} reaches the single-objective-like scale above. When \(g^\dagger=0\), Pareto regret remains well defined, but this particular objective-wise route has no positive separation. %The present theorem does not analyze broader certificates based on multi-coordinate non-dominance witnesses. 
%\textcolor{blue}{MX: meaning of ? multi-coordinate non-dominance witnesses}
%\textcolor{orange}{Already deleted that vague sentence. I was trying to say we are not solving $g^\dagger=0.$}

The key feature of Theorem~\ref{thm:regret} lies in its leading term: apart from the displayed bad-event tail \(2Kd\Delta_{\max}^{\mathrm P}T^{-2}\), there is no explicit multiplicative dependence on the number of objectives, and the first-certificate scale $g^\dagger$ controls the finite-time upper bound for the objective-wise route used by Algorithm~\ref{alg:ucb_lcb_guided}. The dimension dependence in the theorem is therefore confined to a vanishing confidence-failure term, not to the leading exploration cost. Operationally, $g^\dagger$ is not a relaxation of the Pareto benchmark; it is the gap governing the earliest objective-wise certificate after which the same benchmark charges zero regret. Thus, for gap-dependent Pareto regret under selected-arm vector feedback, the multi-objective problem remains single-objective-like at leading order when the first certificate has positive separation.

\begin{remark}[Comparison with Pareto UCB1]
\label{rem:compare_pucb}
Theorem~1 of \citep{drugan2013designing} gives
\(\E[R_T^{\mathrm P}]
\le
\sum_{i\notin A^\star}
\frac{8\log\!(T(d|A^\star|)^{1/4})}{\Delta_i^{\mathrm P}}
+
(1+\frac{\pi^2}{3})\sum_{i\notin A^\star}\Delta_i^{\mathrm P}\).
Consider the generic non-boundary case in which every dominated arm has
positive Pareto gap; zero-gap dominated arms are regret-neutral, and if
\(A^\star=[K]\), the comparison is trivial. Let
\(\Delta_{\min}^{\mathrm P}:=\min_{i\notin A^\star}\Delta_i^{\mathrm P}\).
Then the leading term is at most
\(8(K-|A^\star|)
\log\!(T(d|A^\star|)^{1/4})/\Delta_{\min}^{\mathrm P}\), which can be
written on the scale of Theorem~\ref{thm:regret} as
\((8\cdot \frac{K-|A^\star|}{K}
\cdot \frac{g^\dagger}{\Delta_{\min}^{\mathrm P}}
\cdot
\frac{\log\!(T(d|A^\star|)^{1/4})}{\log T})
\frac{K\log T}{g^\dagger}\).
Thus, the Pareto-UCB1 leading term is organized through all dominated-arm
Pareto gaps and retains \(d\) and \(|A^\star|\) inside the logarithm, whereas
our bound reaches the first-certificate scale \(g^\dagger\). The full
derivation is given in Appendix~\ref{app:comparison_pucb}.

For an exact leading-term comparison, write the Pareto-UCB1 leading term as
\(C_{\mathrm{PUCB}}(T,\mu)K\log T/g^\dagger\), where
\(C_{\mathrm{PUCB}}(T,\mu)
:=
8\cdot \frac{g^\dagger}{K}
(\sum_{i\notin A^\star}\frac{1}{\Delta_i^{\mathrm P}})
\frac{\log\!(T(d|A^\star|)^{1/4})}{\log T}\).
Our coefficient \(64\) is smaller exactly when
\(\frac{g^\dagger}{K}
\sum_{i\notin A^\star}\frac{1}{\Delta_i^{\mathrm P}}
>
8\cdot
\frac{\log T}{\log\!(T(d|A^\star|)^{1/4})}\).
Since
\(\frac{\log T}{\log\!(T(d|A^\star|)^{1/4})}
=
(1+\frac{\log(d|A^\star|)}{4\log T})^{-1}\),
the right-hand side is asymptotically close to \(8\). Hence, the comparison is
driven by how the dominated-arm Pareto gaps relate to \(g^\dagger\). If
\(\Delta_{\min}^{\mathrm P}\ll g^\dagger\), or more generally if
\(\sum_{i\notin A^\star}1/\Delta_i^{\mathrm P}\) is large, the Pareto-UCB1
coefficient can be much larger than \(64\). This is precisely the regime our
method targets: one objective can be certified quickly even when many dominated
arms lie close to the Pareto frontier. If the dominated-arm Pareto gaps are all
of the same order as \(g^\dagger\) and only a few dominated arms remain, then
the Pareto-UCB1 coefficient remains constant-order and may be smaller than
\(64\). This comparison concerns only the leading \(K\log T/g^\dagger\) term,
not the lower-order additive terms.
\end{remark}

\paragraph{Proof Intuition.}

We derive the above upper bound using a new analytical framework. The proof
differs from Pareto-UCB analyses in how regret is accounted for: Pareto UCB
controls regret by bounding the number of times each dominated arm remains an
optimistic Pareto candidate, whereas we focus only on pulls made before the
first Pareto-optimal-arm certificate along the objective with the maximum sub-optimality gap. The bound rests on a single global
confidence event and three deterministic implications on that event:\\
$\qquad$ $\blacktriangleright$ certification stops regret because the certified arm has zero Pareto gap;\\
$\qquad$ $\blacktriangleright$ each non-certifying pull incurs only uncertainty-scale Pareto regret;\\
$\qquad$ $\blacktriangleright$ non-certifying play cannot continue indefinitely as the champion
objective enforces a width floor.

More specifically, Lemma~\ref{lem:certify} shows how a single-coordinate
statistical guarantee controls Pareto regret: a certified unique maximizer of a
single objective already belongs to \(A^\star\). Lemma~\ref{lem:local_charge}
then attributes each non-certifying pull to the confidence radius of the sampled
endpoint in an objective-level top-two race, rather than tracking the arm's role
within an estimated Pareto set. Finally, Lemma~\ref{lem:width_floor} shows that
the easiest separated objective keeps these races active only until the relevant
radii reach the \(g^\dagger\) scale. Thus, the proof focuses entirely on the
certificate witness rather than resolving the Pareto status of individual arms:
we never need to show that every positive-gap arm eventually leaves an
optimistic Pareto set. Once the first certificate is found, the remaining arms
no longer affect Pareto regret. Together, these ingredients attribute all pre-certification regret to confidence
radii, while the width floor converts these terms into a total
\(O(K\log T/g^\dagger)\) bound. The proof is deferred to
Appendix~\ref{app:proofs}.

Beyond its Pareto regret guarantee, the method also identifies a Pareto-optimal
arm through certification-style committed arm selection. This property is not
guaranteed by Pareto UCB, which randomly selects an arm from the estimated
Pareto front. The formal statement is as follows:

\begin{corollary}[High-probability validity of accepted certificates]
\label{cor:certificate_validity}
Run Algorithm~\ref{alg:ucb_lcb_guided} with the theorem-matched confidence radius
\(
\beta_a(t)=\sqrt{\nicefrac{2\log T}{N_a(t-1)}}.
\)
Let \(\widehat a_{\mathrm{cert}}(T)\in[K]\cup\{\emptyset\}\) denote the stored certified leader at the end of horizon \(T\), with \(\widehat a_{\mathrm{cert}}(T)=\emptyset\) if no certificate has been stored. Then
\(
\Pr\!(
\widehat a_{\mathrm{cert}}(T)\neq\emptyset
\ \text{and}\
\widehat a_{\mathrm{cert}}(T)\notin A^\star
)
\le 2Kd\,T^{-3}.
\)
Equivalently, with probability at least \(1-2KdT^{-3}\), every arm accepted by the algorithm is Pareto-optimal.
\end{corollary}

This corollary provides the theoretical analogue of the Pareto-optimal-arm
detection metric in Table~\ref{tab:real-main}, which is further validated empirically.
 %The table reports a validation-tuned empirical implementation and uses a terminal recommendation rule, whereas Corollary~\ref{cor:certificate_validity} applies only to the stored certificate under the conservative theorem-matched radius.

%Section~\ref{sec:computational} examines these comparisons computationally at finite horizons.

\section{Computational Study}
\label{sec:computational}

Our empirical goal is to evaluate finite-horizon performance from two
complementary angles. The real-data benchmark tests outcome metrics under
realistic multi-objective reward distributions and compares the method against
widely used baselines, while the synthetic study provides a stylized validation
of the regret bound through comparison with Pareto-UCB1; we present the
synthetic results in Appendix~\ref{app:synthetic}.

Theoretically, we have validated that Pareto regret is no harder than classical
regret. In addition, the guarantee of identifying an optimal arm also yields the
identification of a Pareto-optimal arm as a byproduct. To the best of our
knowledge, fairness in general multi-objective bandits is primarily an
empirical measure without a formal characterization. Motivated by this, we also
examine this measure numerically. Thus, we report three measures: Pareto regret,
detection of a Pareto-optimal arm, and empirical fairness. The results not only
confirm and align with the theoretical message of the paper, but also reveal
surprising insights and identify a potential hardness associated with distinct
learning goals. We provide additional elaboration in
Appendix~\ref{app:computational}.
%\subsection{Real-World Benchmark: Regret, Detection, and Fairness}

\paragraph{Real-world Datasets.} We evaluate our approach using two real-world datasets. In the Spotify instance, the arms represent music genres, and each pull samples a track from the selected genre. The six objectives correspond to popularity, danceability, energy, loudness, acousticness, and valence. In the KuaiRec instance, the arms are videos, and each objective represents a specific user cohort. The reward is defined as a normalized watch-ratio sample for that particular cohort. The prepared instances contain 114 Spotify arms with $d=6$ objectives and 300 KuaiRec arms with $d=26$ objectives. For both datasets, each evaluated benchmark subset consists of $K=10$ arms. The Spotify instance is built from the Spotify Tracks dataset on Hugging Face~\citep{spotifyTracksDataset}; the dataset card for the mirror used in our scripts lists license \texttt{bsd}. The KuaiRec instance uses KuaiRec 2.0~\citep{gao2022kuairec}; the release used here includes a Creative Commons Attribution-ShareAlike 4.0 International license. 

We select these subsets prior to policy simulation and hold them fixed throughout the evaluation. We categorize the subsets into three empirical-geometry groups. First, \emph{misleading near-front} subsets contain attractive arms that appear close to the Pareto front but actually possess a positive Pareto gap. Second, \emph{friendly near-front} subsets feature near-front arms that have zero or nearly zero gaps, making exploratory mistakes relatively cheap. Finally, \emph{easy separated} subsets offer clear separation with few near-front distractors. We apply the geometry selector directly to the Spotify data, using an independently frozen geometry-selected easy-separated suite for the Spotify easy row. For KuaiRec, we use the same categorizations but employ a contamination-balanced selector to emphasize empirical-front contamination within the geometry groups.

\paragraph{Baselines.} We compare our method against three widely recognized baselines that are commonly used in practice. Pareto UCB1~\citep{drugan2013designing} forms coordinate-wise UCB vectors and samples uniformly from the nondominated set of these optimistic vectors. Annealing-Pareto~\citep{yahyaa2014annealing} uses an \(\epsilon\)-Pareto candidate rule with the reported decay \(\epsilon_s=0.4^s/(Kd)\). Scalarized UCB applies scalar UCB to a fixed finite set of normalized nonnegative linear weights.

\paragraph{Performance Metrics and Implementation.} Pareto regret is defined as before. The terminal recommendation, denoted as $\widehat a_T$, is defined as the arm most frequently played during the final 20\% of rounds. The Pareto set $A^\star$ is derived from the empirical mean vectors of the held-out subset. Additionally, fairness is measured as the cumulative max-min regret:
$
\sum_{t=1}^T
(
\max_a \min_j \mu_a^{(j)}
-
\min_j \mu_{A_t}^{(j)}
).
$ Reported
results are held-out averages over a horizon of \(T=10^6\). The only hyperparameter we tune is this empirical confidence coefficient,
because the theoretical radius is intentionally conservative for an
\emph{upper bound}, rather than for the \emph{actual regret} observed in
numerical experiments. For the baselines, we use the best parameter settings
recommended in their original papers for a fair comparison. Further
implementation and experiment details are available in
Appendix~\ref{app:computational}. The anonymous code and computational artifact are available at
\url{https://anonymous.4open.science/r/momab-width-guided-certification-ACF7}.

\begin{table}[!t]
\centering
\caption{Real-world benchmark on held-out Spotify and KuaiRec subsets at $T=10^6$. Subsets are grouped before policy simulation. The main target is Pareto regret and detection of at least one Pareto-optimal arm; fairness is reported as an empirical measure.}
\label{tab:real-main}
\scriptsize
\setlength{\tabcolsep}{3pt}
\resizebox{\textwidth}{!}{%
\begin{tabular}{llccc}
\toprule
Subset type & Method
& Pareto regret $\downarrow$
& $\Pr(\widehat a_T\in A^\star)\uparrow$
& Fairness regret $\downarrow$ \\
\midrule
\multicolumn{5}{l}{\textbf{Spotify d6}} \\
Misleading near-front & Width-guided & $\mathbf{3.05\pm1.99}$ & $\mathbf{1.000}$ & $1.27\times 10^{5} \pm 8.86\times 10^{4}$ \\
Misleading near-front & Pareto UCB1 & $2584.77\pm60.31$ & $\mathbf{1.000}$ & $1.67\times 10^{5} \pm 1.10\times 10^{4}$ \\
Misleading near-front & Annealing-Pareto & $337.59\pm1504.82$ & 0.983 & $1.37\times 10^{5} \pm 2.53\times 10^{4}$ \\
Misleading near-front & Scalarized UCB & $1435.03\pm258.58$ & $\mathbf{1.000}$ & $\mathbf{1.07}\times 10^{5} \pm \mathbf{2.14}\times 10^{4}$ \\
\midrule
Friendly near-front & Width-guided & $\mathbf{0.12\pm0.07}$ & $\mathbf{1.000}$ & $\mathbf{1.40}\times 10^{5} \pm \mathbf{9.26}\times 10^{4}$ \\
Friendly near-front & Pareto UCB1 & $212.44\pm61.38$ & 0.967 & $2.16\times 10^{5} \pm 2.44\times 10^{4}$ \\
Friendly near-front & Annealing-Pareto & $26.81\pm71.75$ & 0.983 & $2.05\times 10^{5} \pm 3.62\times 10^{4}$ \\
Friendly near-front & Scalarized UCB & $65.10\pm39.01$ & $\mathbf{1.000}$ & $1.55\times 10^{5} \pm 2.62\times 10^{4}$ \\
\midrule
Easy separated & Width-guided & $\mathbf{0.69\pm0.43}$ & $\mathbf{1.000}$ & $2.23\times 10^{5} \pm 5.89\times 10^{4}$ \\
Easy separated & Pareto UCB1 & $1916.61\pm57.93$ & $\mathbf{1.000}$ & $1.83\times 10^{5} \pm 4.15\times 10^{4}$ \\
Easy separated & Annealing-Pareto & $408.06\pm1935.09$ & 0.983 & $1.79\times 10^{5} \pm 4.64\times 10^{4}$ \\
Easy separated & Scalarized UCB & $1198.16\pm88.62$ & $\mathbf{1.000}$ & $\mathbf{1.25}\times 10^{5} \pm \mathbf{2.97}\times 10^{4}$ \\
\midrule
\multicolumn{5}{l}{\textbf{KuaiRec d26}} \\
Misleading near-front & Width-guided & $\mathbf{8.00\pm4.41}$ & $\mathbf{1.000}$ & $2.85\times 10^{4} \pm 3.76\times 10^{4}$ \\
Misleading near-front & Pareto UCB1 & $4772.64\pm1766.07$ & $\mathbf{1.000}$ & $5.03\times 10^{4} \pm 6.82\times 10^{3}$ \\
Misleading near-front & Annealing-Pareto & $13.36\pm8.80$ & $\mathbf{1.000}$ & $2.80\times 10^{4} \pm 2.04\times 10^{4}$ \\
Misleading near-front & Scalarized UCB & $2158.68\pm585.78$ & $\mathbf{1.000}$ & $\mathbf{1.95}\times 10^{4} \pm \mathbf{4.95}\times 10^{3}$ \\
\midrule
Friendly near-front & Width-guided & $\mathbf{0.19\pm0.14}$ & $\mathbf{1.000}$ & $\mathbf{1.67}\times 10^{4} \pm \mathbf{2.53}\times 10^{4}$ \\
Friendly near-front & Pareto UCB1 & $200.54\pm44.32$ & 0.950 & $4.67\times 10^{4} \pm 1.76\times 10^{4}$ \\
Friendly near-front & Annealing-Pareto & $3.03\pm6.95$ & $\mathbf{1.000}$ & $4.22\times 10^{4} \pm 1.66\times 10^{4}$ \\
Friendly near-front & Scalarized UCB & $47.52\pm7.27$ & $\mathbf{1.000}$ & $2.27\times 10^{4} \pm 4.91\times 10^{3}$ \\
\midrule
Easy separated & Width-guided & $\mathbf{7.53\pm3.08}$ & $\mathbf{1.000}$ & $\mathbf{0.012}\times 10^{3} \pm \mathbf{0.005}\times 10^{3}$ \\
Easy separated & Pareto UCB1 & $2855.97\pm455.90$ & $\mathbf{1.000}$ & $6.33\times 10^{3} \pm 2.02\times 10^{3}$ \\
Easy separated & Annealing-Pareto & $8.66\pm5.15$ & $\mathbf{1.000}$ & $0.015\times 10^{3} \pm 0.009\times 10^{3}$ \\
Easy separated & Scalarized UCB & $1789.05\pm175.84$ & $\mathbf{1.000}$ & $3.31\times 10^{3} \pm 0.538\times 10^{3}$ \\
\bottomrule
\end{tabular}%
}
\end{table}

\iffalse
\textcolor{blue}{MX: it seems that our fairness is not always the worst, would it be possible to bold the places where our fairness is smallest (also bold pareto arm identification), and then explain in sentences, saying that we may sacrifice fairness but in some cases may achieves best of all worlds in terms all three metrics. This may also motivate future focus on fairness, which may reflect a new hardness.}
\textcolor{orange}{I think we kind of said that in the Experiment Results section so I guess we can safely remove my next orange sentence?}
\textcolor{orange}{Boldface in the detection column marks all methods tied for the largest terminal detection rate, and boldface in the fairness column marks the smallest average fairness within each dataset--subset block. The results show that Width-guided does not uniformly sacrifice fairness. In Spotify misleading, Spotify easy, and KuaiRec misleading, it has the smallest Pareto regret and perfect terminal detection, but a scalarized policy has smaller fairness. In Spotify friendly, KuaiRec friendly, and KuaiRec easy, Width-guided is best or tied-best on all three reported metrics. These mixed outcomes suggest that fairness introduces a separate learning goal, and understanding when it aligns with Pareto regret is an important direction for future work.}
\fi 

\paragraph{Experiment Results.} Table~\ref{tab:real-main} reports the full results.  Across these held-out real-world subsets, Width-guided achieves the lowest Pareto regret in every Spotify and KuaiRec group under the stated hyperparameter protocol. The performance gap is most pronounced in the misleading near-front and easy separated groups. In these settings, repeatedly playing positive-gap arms proves costly for methods that attempt to maintain or exploit a broader empirical front. While the friendly groups also show favorable results for our method, the interpretation differs slightly. Because near-front mistakes in these subsets are cheap, the absolute scale of regret remains small across all algorithms. Measured against the best non-Width-guided Pareto-regret mean in each block, Width-guided is lower by about two or more orders of magnitude in all three Spotify groups (\(\mathbf{111\times}\), \(\mathbf{223\times}\), and \(\mathbf{591\times}\)),
by more than one order of magnitude in KuaiRec friendly
\(\mathbf{16\times}\), and by smaller factors in KuaiRec misleading and
KuaiRec easy (\(\mathbf{1.7\times}\) and \(\mathbf{1.2\times}\)).  Thus the main pattern is order-of-magnitude regret reduction in four of six groups, with two closer KuaiRec comparisons. These scale differences are intrinsic to the dataset structures themselves; the performance gap is especially large in the Spotify groups and in KuaiRec friendly, while KuaiRec misleading and KuaiRec easy show closer comparisons. Our algorithm applies regardless and can actually identify such structures based on the results. Furthermore, Width-guided has terminal detection frequency \(1.000\) in every reported group, so the low-regret rows also select \(\widehat a_T\in A^\star\) at the terminal recommendation stage.

The fairness column highlights the complementary nature of these learning
objectives. In some instances (e.g., Spotify misleading, Spotify easy, and KuaiRec misleading), although Scalarized UCB suffers from
substantially larger Pareto regret, it achieves smaller max-min fairness. This
behavior is consistent with the role of first certification: our approach is
designed to minimize Pareto regret, rather than to ensure full-front
coverage or fairness-style balancing. In other instances (e.g., Spotify friendly, KuaiRec friendly, and KuaiRec easy), however, our method is \emph{best-of-all-worlds on all three reported metrics}: Pareto regret, terminal Pareto-optimal-arm detection, and fairness. 

Taken together, these empirical results highlight the practical utility of our
method, besides complementing the theoretical guarantees. They also
suggest that fairness deserves formal study and may reflect a new source of
hardness, rather than being treated only as an empirical measure.

\bibliographystyle{abbrv}
\bibliography{refs}

\newpage

\appendix

\section{Technical appendices and supplementary material}
%\textcolor{blue}{MX: an outline like an outline in a book?}

The appendix is organized as follows.

\begin{enumerate}
    \item \textbf{Appendix~\ref{sec:related}: Related Work.}
    Supports Section~\ref{sec:introduction}.

    \item \textbf{Appendix~\ref{app:synthetic}: Synthetic Experiments.}
    Supports Section~\ref{sec:computational}.

    \item \textbf{Appendix~\ref{app:computational}: Reproducibility Details.}
    Provides implementation details for Section~\ref{sec:computational}.

    \item \textbf{Appendix~\ref{app:theory}: Additional Theory.}
    Provides additional theoretical statements for Section~\ref{sec:results}.

    \item \textbf{Appendix~\ref{app:proofs}: Proofs.}
    Contains the proofs of the results in Section~\ref{sec:results}.
\end{enumerate}

\subsection{Related Work}
\label{sec:related}

\textbf{Stochastic Multi-objective Bandits.} Multi-objective multi-armed bandits (MO-MAB) generalize the classical MAB framework by replacing the scalar reward with a vector-valued reward, which substantially complicates both algorithm design and performance analysis. A natural starting point is to extend confidence-bound methods from MAB to the multi-objective setting. In this direction, Pareto UCB~\citep{drugan2013designing} maintains coordinate-wise upper confidence bounds for each arm and identifies the estimated Pareto front rather than a single best arm; an arm is then selected uniformly at random from this set. Using a multi-dimensional Chernoff--Hoeffding bound, Pareto UCB achieves logarithmic Pareto regret, and subsequent refinements further improve its performance~\citep{drugan2014pareto}. However, it is shown that the regret is monotonically increasing in the dimensionality, which is reversed by the finding in this paper. Indeed, the key quantity is the suboptimality gap, and Pareto regret can actually be minimized by the minimum marginal regret, rather than necessarily growing with the dimensionality. Other related variants have also been studied for stochastic MO-MAB, including PF-LEX, which incorporates lexicographic order and satisficing objectives~\citep{huyuk2021multi}, whereas we consider the standard Pareto order relation.

A different stream of related work is rooted in the multi-objective optimization (MOO) literature. A number of algorithms developed in that area have demonstrated strong empirical effectiveness in practical settings, including evolutionary methods for large-scale problems that explicitly navigate the exploration--exploitation trade-off~\citep{hong2021evolutionary}, as well as Annealing Pareto Knowledge Gradient, which resembles a Thompson-sampling-type strategy~\citep{yahyaa2014annealing}. Despite these promising empirical results, such methods generally do not come with regret guarantees, which limits their applicability in online learning settings where theoretical reliability is important. Another common strategy inspired by MOO is to apply scalarization, for example through linear or Chebyshev functions, so that the original vector-valued problem is converted into a standard scalar bandit problem. Although this simplification is often computationally convenient, it may obscure important structure in the multi-objective reward and tends to depend heavily on the specific choice of scalarization. Closely related is the line of work that optimizes the General Gini Index (GGI) via online convex optimization~\citep{busa2017multi}; however, the associated regret is measured relative to GGI, rather than with respect to Pareto performance.

\textbf{Contextual Multi-objective Bandits.} For \emph{contextual} multi-objective bandits, the literature has mostly focused on the linear contextual bandit case, where the reward vector is a linear function of contexts or feature vectors. One representative example is MO-LinUCB~\citep{mehrotra2020bandit}, which extends linear contextual bandits to vector-valued rewards. This line of work highlights that, once context is introduced, the challenge is no longer only to learn the Pareto structure across arms, but also to adapt that structure to context-dependent reward models. Recent exceptions include the contextual case in \citep{turgay2018multi}, where the linear function is replaced by a similarity distance, meaning that the reward mean is Lipschitz continuous in the context, and \citep{wanigasekara2019learning}, where this assumption is removed but the theoretical guarantee is also unknown, as well as neural bandits~\citep{cheng2025multi}. We focus on the classical stochastic bandit case, while uncovering the effect of dimensionality and achieving optimal regret in that setting as well.

\textbf{Adversarial Multi-objective Bandits.} Recently, adversarial multi-objective bandits have attracted increasing attention because they generalize naturally to non-stochastic reward vectors. \citep{xu2023pareto} introduces the first formulation of Pareto regret that directly measures the one-shot distance between the cumulative received reward vector and the Pareto front induced by the cumulative reward vectors of all arms. They show the surprising result that this notion of Pareto regret is smaller than any marginal regret, and then propose a best-of-both-worlds algorithm that works in both stochastic and adversarial settings. Notably, our notion of regret differs from theirs: we consider cumulative distance rather than one-shot distance, so their results do not directly apply here. In addition, we establish an algorithm for achieving minimum marginal regret, whereas their best-of-both-worlds guarantee is only with respect to \(T\), not the dimension. We emphasize that the subtle connection they uncover between their notion of Pareto regret and marginal regret is what inspires our approach here, although the settings and regret notions differ substantially, as described above.

\textbf{Multi-objective Optimization.} In a broader class of multi-objective decision problems under uncertainty, one common approach is to encode preferences through utility functions, whether linear or nonlinear, as is often done in economics. Conceptually, these methods serve a similar purpose to scalarization by converting vector-valued outcomes into a single evaluative criterion. A more recent direction is the relative minimax regret framework of \citep{groetzner2022multiobjective}, which comes with theoretical guarantees. Their method evaluates each dimension against the worst-case realization of uncertainty and then optimizes decisions according to Pareto dominance. Despite its theoretical appeal, the framework requires all optimal values under the possible realizations of uncertainty to be computed in advance, making it difficult to adapt to online bandit problems where decisions must be taken sequentially without such offline precomputation. Indeed, there is classical work on using online bandits to solve MOO~\citep{huang2024direct}. Another relevant line of work is offline multi-objective bandits~\citep{cheng2026offline}, which is designed for the offline setting, whereas our focus is online, highlighting the main difference.

\subsection{Numerical Results on Synthetic Datasets}\label{app:synthetic}

We also include a controlled synthetic family to isolate the mechanism behind the comparison with Pareto UCB1.  This experiment uses the same horizon and Monte Carlo budget as the real-data benchmark: \(T=10^6\) and \(20\) repetitions for every setting.  Unlike the real-data Width-guided implementation in Table~\ref{tab:real-main}, the synthetic experiment uses the theorem-matched confidence radius \(\sqrt{2\log T/N_a(t-1)}\), with no empirical coefficient tuning.

Each synthetic instance has \(K=20\) arms and \(d=2\) objectives.  Two arms form the Pareto front, with means \((p+g,p)\) and \((p,p+g)\), where \(p=0.25\) and \(g=0.55\).  A crowd of \(m\) dominated arms has mean \((p-\eta,p+g-\delta)\), with \(\eta=0.20\), and the remaining arms are low-reward fillers.  Thus the first-certification scale remains fixed at \(g^\dagger=0.55\), while the Pareto-UCB burden grows either by decreasing the dominated-arm Pareto gap \(\delta=\Delta_{\min}^{\mathrm P}\) or by increasing the number \(m\) of near-front dominated arms.

\begin{figure}[t]
\centering
\includegraphics[width=\textwidth]{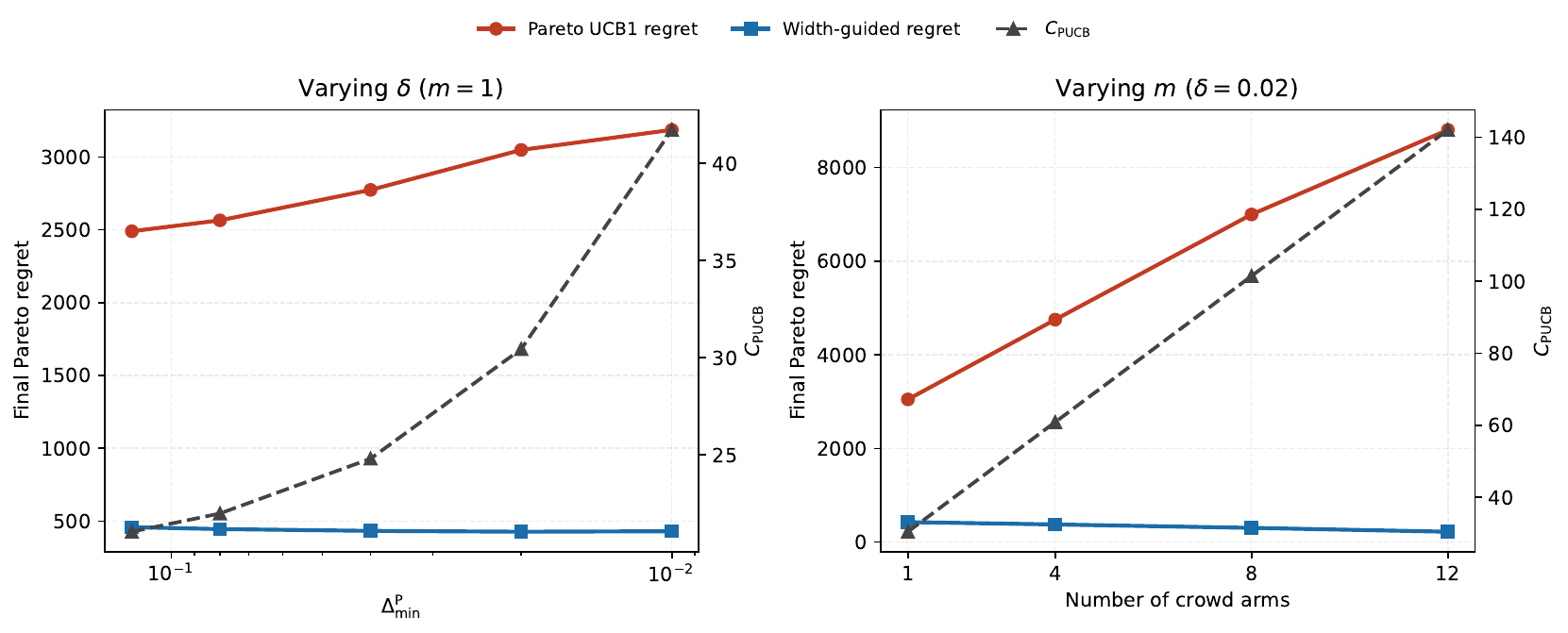}
\caption{Synthetic controlled family at \(T=10^6\) with \(20\) repetitions.  The first panel decreases the Pareto gap \(\delta\) with one near-front dominated arm; the second increases the number \(m\) of near-front dominated arms at fixed \(\delta=0.02\).  The first-certification scale \(g^\dagger=0.55\) is fixed throughout.}
\label{fig:synthetic-comparison}
\end{figure}

\begin{table}[t]
\centering
\caption{Synthetic family with fixed certification gap and small Pareto gaps at \(T=10^6\); \(g^\dagger=0.55\) throughout and \(\delta=\Delta_{\min}^{\mathrm P}\). Regret entries report mean \(\pm\) standard deviation over \(20\) runs.}
\label{tab:neurips-pucb-vs-width}
\scriptsize
\setlength{\tabcolsep}{4pt}
\begin{tabular}{@{}cccccc@{}}
\toprule
\(\delta\) & \(m\) & \(C_{\mathrm{PUCB}}\) & Pareto UCB1 & Width-guided & Cert. rate \\
\midrule
\multicolumn{6}{@{}l}{\textit{Panel A: vary \(\delta\), \(m=1\)}} \\
0.120 & 1 & 21.05 & \(2490.0 \pm 46.8\) & \(458.5 \pm 27.2\) & 100\% \\
0.080 & 1 & 21.99 & \(2564.8 \pm 70.3\) & \(445.6 \pm 28.7\) & 100\% \\
0.040 & 1 & 24.81 & \(2773.7 \pm 97.3\) & \(433.2 \pm 26.3\) & 100\% \\
0.020 & 1 & 30.45 & \(3048.1 \pm 104.4\) & \(427.4 \pm 23.1\) & 100\% \\
0.010 & 1 & 41.72 & \(3185.8 \pm 146.7\) & \(430.6 \pm 30.9\) & 100\% \\
\midrule
\multicolumn{6}{@{}l}{\textit{Panel B: vary \(m\), \(\delta=0.020\)}} \\
0.020 & 4 & 60.89 & \(4751.0 \pm 236.1\) & \(377.0 \pm 19.8\) & 100\% \\
0.020 & 8 & 101.48 & \(6997.2 \pm 395.6\) & \(304.0 \pm 12.9\) & 100\% \\
0.020 & 12 & 142.08 & \(8803.7 \pm 343.8\) & \(221.7 \pm 10.7\) & 100\% \\
\bottomrule
\end{tabular}%
\end{table}

\begin{figure}[t]
\centering
\includegraphics[width=\textwidth]{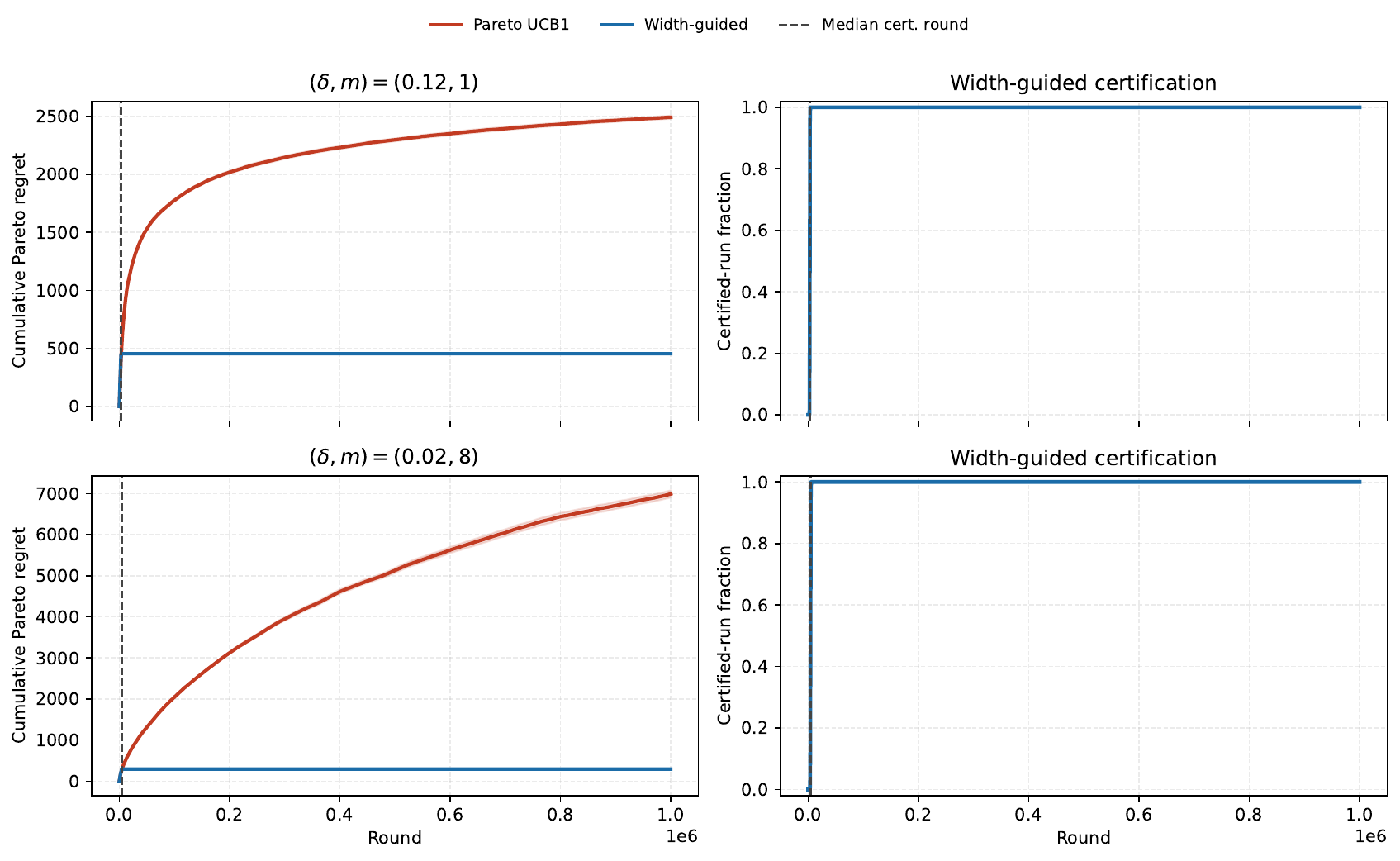}
\caption{Synthetic regret and certification trajectories for two representative settings.  Width-guided certifies in every run and then stops accumulating Pareto regret, while Pareto UCB1 continues to sample from a broader optimistic front.}
\label{fig:synthetic-trajectories}
\end{figure}

\begin{remark}[Synthetic view of the \(C_{\mathrm{PUCB}}\) comparison]
The coefficient \(C_{\mathrm{PUCB}}\) is the normalized leading constant appearing in the comparison with the Pareto-UCB1 gap-dependent bound.  The sufficient comparison condition \(64<C_{\mathrm{PUCB}}(T,\mu)\) from Section~\ref{sec:results} is intentionally conservative, not a necessary condition for finite-horizon improvement.  In Table~\ref{tab:neurips-pucb-vs-width}, Width-guided already has smaller regret in the smaller-\(C_{\mathrm{PUCB}}\) settings, and the advantage becomes more pronounced as \(C_{\mathrm{PUCB}}\) increases through smaller \(\delta\) or larger \(m\).  This is the synthetic counterpart of the real-data pattern: when many positive-gap arms remain attractive near the empirical front, first certification can avoid repeated positive-gap pulls.
\end{remark}

\subsection{Additional Computational Details}
\label{app:computational}

This appendix records reproducibility details for the computational study in Section~\ref{sec:computational}.  It specifies the computational resources used, the real-data benchmark construction, validation/test split, policy definitions, selected Width-guided coefficient, and front-identification diagnostics.  It also records the matched-budget synthetic-family setup.  The main real-data table reports outcome metrics only; the diagnostics here are included to clarify what the first-certification objective does and does not optimize.

\paragraph{Computer resources.} The final evaluation scripts in the open-source artifact were rerun as CPU-only Slurm array jobs; no GPUs were requested or used.  For every final real-data and synthetic evaluation, we used \(20\) array shards, and each shard requested \(20\) CPU cores, \(40\)~GB of memory, and a conservative \(12\)-hour wall-clock limit.  The rerun outputs are included in the open-source artifact.  In this rerun, all completed final-evaluation shards finished within \(6\) minutes of wall-clock time, excluding scheduler queueing.  The software environment used Python 3.9.21, NumPy 2.0.2, and Matplotlib 3.9.4.

\paragraph{Dataset construction.}
The held-out benchmark uses two prepared real-data instances.  In Spotify, the \(114\) arms are genres, \(d=6\), and rewards are empirical track feature vectors sampled with replacement from the selected genre.  In KuaiRec, the \(300\) arms are videos, \(d=26\), and rewards are cohort-specific normalized watch-ratio samples.  All reward coordinates are scaled to \([0,1]\), and each benchmark subset contains \(K=10\) arms.  The Spotify objectives are popularity, danceability, energy, loudness, acousticness, and valence.  The KuaiRec objectives are the \(26\) largest retained user cohorts formed from user activity, follow-count range, and registration-age range.  Spotify subsets use the geometry selector, while KuaiRec uses the contamination-balanced geometry selector, which applies the same misleading/friendly/easy labels but prioritizes candidates where bootstrapped empirical fronts contain positive-gap near-front arms.

\paragraph{Validation, coefficient selection, and held-out evaluation.}
For Width-guided, the only hyperparameter selected on our validation subsets is the confidence coefficient \(c\) in
\[
\beta_a(t)=\sqrt{\frac{c\log T}{N_a(t-1)}} .
\]
The analyzed theorem corresponds to the conservative choice \(c=2\).  For the real-data benchmark, \(c\) is selected on validation subsets at the same horizon \(T=10^6\), and the selected value is frozen before held-out testing.  This empirical implementation should not be read as carrying the high-probability guarantee of Theorem~\ref{thm:regret}.  Baseline hyperparameters are taken from the recommended settings in their original papers rather than reselected on our validation subsets.

\begin{table}[ht]
\centering
\caption{Validation/test split and selected Width-guided coefficient.  ``Subsets \(\times\) seeds'' is per dataset and subset type.  All held-out rows use the same test seed; Spotify easy uses a separately frozen geometry-selected easy suite.}
\label{tab:app-real-tuning}
\scriptsize
\setlength{\tabcolsep}{3pt}
\resizebox{\textwidth}{!}{%
\begin{tabular}{lccccc}
\toprule
Dataset & Validation seed & Test seed & Validation & Test & Selected \(c\) \\
\midrule
Spotify d6, misleading/friendly & 20260604 & 20260704 & \(3\times3\) & \(3\times20\) & \(0.02\) \\
Spotify d6, easy & 20260604 & 20260704 & \(3\times3\) & \(3\times20\) & \(0.00894\) \\
KuaiRec d26, all types & 20260604 & 20260704 & \(3\times3\) & \(3\times20\) & \(0.02\) \\
\bottomrule
\end{tabular}%
}
\end{table}

The held-out table therefore averages \(3\) subsets and \(20\) reward seeds per subset type, for \(60\) runs per method and type.  Candidate subset generation uses frozen subset files selected before the policy simulations; the Spotify easy row uses a separately frozen geometry-selected easy suite.  The validation and test seeds are different, and the same held-out test seed is used across rows, so no held-out reward stream is used when choosing \(c\).

\paragraph{Policy implementations.}
Within a Monte Carlo run, all policies use the same reward stream for each arm.  Width-guided is Algorithm~\ref{alg:ucb_lcb_guided} with the selected coefficient above; after an objective-wise certificate appears, it commits to the certified arm.  Pareto UCB1 pulls each arm once, forms coordinate-wise upper confidence vectors, and samples uniformly from the optimistic empirical Pareto set.  Annealing-Pareto uses the reported empirical decay setting from its original experiments and otherwise samples uniformly from the empirical Pareto set.  Scalarized UCB applies ordinary UCB over a fixed multi-weight scalarization set.

\paragraph{Front-identification and coverage diagnostics.}
Table~\ref{tab:app-real-front-coverage} reports diagnostics on the same held-out runs.  Precision and recall compare the Pareto front of the final empirical mean estimates to the true empirical front of the held-out subset.  Coverage entropy is the normalized entropy of total pull counts over the true Pareto arms.  These metrics evaluate full-front learning and diversity of play, not terminal Pareto-optimal-arm detection.

\begin{table}[ht]
\centering
\caption{Front-identification and coverage diagnostics on the held-out real-data subsets.}
\label{tab:app-real-front-coverage}
\tiny
\setlength{\tabcolsep}{2pt}
\resizebox{\textwidth}{!}{%
\begin{tabular}{lllccc}
\toprule
Dataset & Subset type & Method & Front precision \(\uparrow\) & Front recall \(\uparrow\) & Coverage entropy \(\uparrow\) \\
\midrule
Spotify d6 & Misleading & Width-guided & 0.760 & 0.947 & 0.001 \\
Spotify d6 & Misleading & Pareto UCB1 & 1.000 & 1.000 & 1.000 \\
Spotify d6 & Misleading & Annealing-Pareto & 0.782 & 0.840 & 0.848 \\
Spotify d6 & Misleading & Scalarized UCB & 1.000 & 1.000 & 0.812 \\
\midrule
Spotify d6 & Friendly & Width-guided & 0.850 & 0.921 & 0.002 \\
Spotify d6 & Friendly & Pareto UCB1 & 0.971 & 1.000 & 1.000 \\
Spotify d6 & Friendly & Annealing-Pareto & 0.879 & 0.835 & 0.820 \\
Spotify d6 & Friendly & Scalarized UCB & 0.962 & 1.000 & 0.775 \\
\midrule
Spotify d6 & Easy & Width-guided & 0.713 & 0.919 & 0.000 \\
Spotify d6 & Easy & Pareto UCB1 & 1.000 & 0.972 & 1.000 \\
Spotify d6 & Easy & Annealing-Pareto & 0.852 & 0.881 & 0.838 \\
Spotify d6 & Easy & Scalarized UCB & 1.000 & 0.964 & 0.733 \\
\midrule
KuaiRec d26 & Misleading & Width-guided & 0.416 & 0.993 & 0.001 \\
KuaiRec d26 & Misleading & Pareto UCB1 & 1.000 & 0.990 & 0.667 \\
KuaiRec d26 & Misleading & Annealing-Pareto & 0.925 & 0.940 & 0.620 \\
KuaiRec d26 & Misleading & Scalarized UCB & 1.000 & 0.987 & 0.447 \\
\midrule
KuaiRec d26 & Friendly & Width-guided & 0.815 & 0.950 & 0.003 \\
KuaiRec d26 & Friendly & Pareto UCB1 & 0.963 & 1.000 & 1.000 \\
KuaiRec d26 & Friendly & Annealing-Pareto & 0.910 & 0.912 & 0.934 \\
KuaiRec d26 & Friendly & Scalarized UCB & 0.961 & 1.000 & 0.737 \\
\midrule
KuaiRec d26 & Easy & Width-guided & 0.502 & 1.000 & 0.000 \\
KuaiRec d26 & Easy & Pareto UCB1 & 1.000 & 1.000 & 0.000 \\
KuaiRec d26 & Easy & Annealing-Pareto & 1.000 & 1.000 & 0.000 \\
KuaiRec d26 & Easy & Scalarized UCB & 1.000 & 1.000 & 0.000 \\
\bottomrule
\end{tabular}%
}
\end{table}

The diagnostics show the expected separation between objectives.  Pareto UCB1 and Annealing-Pareto often have better front precision, recall, or coverage entropy because they continue sampling across the empirical front.  Width-guided concentrates on one certified Pareto-optimal arm, which is why its coverage entropy is close to zero.  This behavior is a feature for Pareto regret but not for full-front recovery.

\paragraph{Synthetic-family reproducibility.}
The synthetic family in Figure~\ref{fig:synthetic-comparison}, Table~\ref{tab:neurips-pucb-vs-width}, and Figure~\ref{fig:synthetic-trajectories} uses the same horizon and repetition count as the real-data benchmark: \(T=10^6\), \(20\) final-regret repetitions per setting, and \(20\) trajectory repetitions per plotted setting.  Within each setting, all policies are evaluated on the same synthetic instance family and matched Monte Carlo repetitions.  Unlike the real-data benchmark, the synthetic Width-guided runs use the theorem-matched radius \(\sqrt{2\log T/N_a(t-1)}\), with no validation-tuned coefficient.

\subsection{Additional Theoretical Results}\label{app:theory}
Because the confidence radius depends only on the realized sample count, the natural concentration event is a single horizon-$T$ event that covers every arm, every objective, and every realized sample count up to time $T$.

\begin{proposition}[Global confidence event]
\label{prop:confidence}
Let
\[
\Omega_T :=
\left\{
\left|\widehat{\mu}_a^{(j)}(t-1)-\mu_a^{(j)}\right| \le \beta_a(t)
\quad \forall K<t\le T,\; a \in [K],\; j \in [d]
\right\}.
\]
Then
\[
\Pr(\Omega_T^c) \le 2Kd\,T^{-3}.
\]
\end{proposition}

Proposition~\ref{prop:confidence} resolves the concentration step. Its role is to reduce the adaptive sampling process to a single event $\Omega_T$ on which every empirical mean lies inside its current confidence interval, simultaneously across all arms, objectives, and rounds. Once we condition on $\Omega_T$, the remaining work is deterministic and algorithmic: we must show that certification returns a Pareto-optimal arm, that each pull at a non-certifying round incurs only uncertainty-scale Pareto regret, and that at every non-certifying round the champion objective still has width of order $g^\dagger$.

\begin{lemma}[Certification identifies a Pareto-optimal arm]
\label{lem:certify}
Fix a round $K<t\le T$ and an objective $j$. On the event $\Omega_T$, if
\[
L_{b_j(t)}^{(j)}(t) > U_{c_j(t)}^{(j)}(t),
\]
then $b_j(t)$ is the unique maximizer of objective $j$, which implies that $b_j(t)\in A^\star$ and hence $\Delta_{b_j(t)}^{\mathrm P}=0$.
\end{lemma}

Lemma~\ref{lem:certify} is the point where a one-coordinate statistical statement becomes a Pareto regret statement. The algorithm does not certify the full Pareto front; it certifies that one arm is the winner of one objective, and therefore lies in \(A^\star\). Once certification occurs, the stored leader can be exploited forever because every Pareto-optimal arm has zero Pareto gap.

The remaining main task is to control the regret accumulated before certification. All positive Pareto regret comes from non-certifying rounds, that is, rounds at which no certified leader has yet been stored and no objective certifies. The key bridge is that the Pareto gap is upper-bounded by the deficit to the best arm on any single objective. Thus, once the selected objective is fixed, controlling an endpoint's objective-wise deficit controls its Pareto regret, even though the benchmark itself is not scalarized. At a non-certifying round, the algorithm selects the objective with the largest pair width and samples the more uncertain endpoint of its top-two UCB pair. The next lemma shows that every such pull has Pareto regret of only confidence-radius scale.

\begin{lemma}[Non-certifying pulls have uncertainty-scale regret]
\label{lem:local_charge}
Fix a round $K<t\le T$ on $\Omega_T$. Suppose round $t$ is non-certifying. Let
\[
j_t \in \arg\max_{j\in[d]} W_j(t),
\qquad
A_t \in \arg\max_{a\in\{b_{j_t}(t),\,c_{j_t}(t)\}} \beta_a(t).
\]
Then
\[
\Delta_{A_t}^{\mathrm P} \le 4\beta_{A_t}(t).
\]
\end{lemma}

Lemma~\ref{lem:local_charge} is the local regret-to-uncertainty bridge. At a non-certifying round, the sampled arm is one endpoint of an unresolved top-two UCB pair. If that arm had a large Pareto gap, then on the selected objective it would lie too far below the true winner; on $\Omega_T$, that would either force the true winner to outrank it in the UCB ordering or make the pair certify already. Hence the chosen arm can remain active only because uncertainty still masks its deficit, and its instantaneous Pareto regret is controlled by its current confidence radius.

By itself, Lemma~\ref{lem:local_charge} does not yet imply a cumulative regret bound. A complementary ingredient is needed: before certification occurs, the champion objective must still have unresolved width of order $g^\dagger$, and because the algorithm selects the widest pair, every sampled arm at a non-certifying round must still have radius at least of order $g^\dagger$. This is where the scheduling rule matters. The learner does not know in advance which objective will provide the first certificate, but as long as no certificate exists, the champion objective remains unresolved at scale \(g^\dagger\). Choosing the widest race transfers this lower bound to the race actually sampled.

\begin{lemma}[Champion width floor at non-certifying rounds]
\label{lem:width_floor}
Fix a round $K<t\le T$ on $\Omega_T$ and suppose $g^\dagger>0$. Suppose round $t$ is non-certifying. Then the champion objective satisfies
\[
\beta_{b_{j^\dagger}(t)}(t)+\beta_{c_{j^\dagger}(t)}(t)\ge \frac{g^\dagger}{2}.
\]
Consequently, if
\[
j_t \in \arg\max_{j\in[d]} W_j(t),
\qquad
A_t \in \arg\max_{a\in\{b_{j_t}(t),\,c_{j_t}(t)\}} \beta_a(t),
\]
with
\[
b_t := b_{j_t}(t), \qquad c_t := c_{j_t}(t),
\]
then
\[
\beta_{b_t}(t)+\beta_{c_t}(t)\ge \frac{g^\dagger}{2},
\qquad
\beta_{A_t}(t)\ge \frac{g^\dagger}{4}.
\]
\end{lemma}

Lemma~\ref{lem:width_floor} supplies the global complement to Lemma~\ref{lem:local_charge}. Lemma~\ref{lem:local_charge} says that at a non-certifying round, regret can be charged to the current confidence radius of the sampled arm. Lemma~\ref{lem:width_floor} says that as long as certification has not yet happened, those radii cannot collapse to an arbitrarily small scale: the champion objective still forces unresolved width of order $g^\dagger$, so every sampled arm must still have radius at least of order $g^\dagger$.

The main theorem then follows by combining these two facts: each pre-certification pull costs only radius-scale regret, each arm can incur only finitely many such pulls before its radius drops below the level allowed by Lemma~\ref{lem:width_floor}, and after certification the regret becomes identically zero.

\subsection{Proofs}
\label{app:proofs}

\label{app:process_details}

Let
\[
\cF_t := \sigma(A_1,X_1,\ldots,A_t,X_t)
\]
be the natural filtration. The policy is non-anticipatory: $A_{t+1}$ is $\cF_t$-measurable. Under the latent-sequence construction, if $N_a(t-1)\ge 1$, then the empirical mean used by the algorithm is exactly
\[
\widehat{\mu}_a^{(j)}(t-1)
:=
\frac{1}{N_a(t-1)}
\sum_{s=1}^{t-1} X_s^{(j)} \mathbf{1}\{A_s=a\}
=
\frac{1}{N_a(t-1)}
\sum_{n=1}^{N_a(t-1)} Y_{a,n}^{(j)}.
\]

\subsubsection{Proof of Proposition \ref{prop:zero_gap_not_pareto}}

\begin{proof}
Let \(\mu_1=(1,1)\) and \(\mu_2=(1,0)\). Arm \(1\) dominates arm \(2\), so
\(2\notin A^\star\). However, for every \(\epsilon>0\),
\[
\mu_2+\epsilon\mathbf{1}=(1+\epsilon,\epsilon)
\]
is not dominated by \(\mu_1\), because its first coordinate is larger than that
of \(\mu_1\). Hence \(\Delta_2^{\mathrm P}=0\).
\end{proof}

\subsubsection{Proof of Corollary~\ref{cor:certificate_validity}}

\begin{proof}
On \(\Omega_T\), Lemma~\ref{lem:certify} shows that whenever the certification condition
\[
L_{b_j(t)}^{(j)}(t)>U_{c_j(t)}^{(j)}(t)
\]
holds, the certified leader \(b_j(t)\) is the unique maximizer of objective \(j\). Hence every other arm is strictly worse than \(b_j(t)\) on coordinate \(j\), so no other arm can Pareto dominate it. Therefore \(b_j(t)\in A^\star\). Since Algorithm~\ref{alg:ucb_lcb_guided} stores only arms accepted through this certificate condition, a stored non-Pareto-optimal certified leader can occur only on \(\Omega_T^c\). Proposition~\ref{prop:confidence} gives \(\Pr(\Omega_T^c)\le 2KdT^{-3}\), proving the claim.
\end{proof}

\subsubsection{Proof of Proposition~\ref{prop:confidence}}

\begin{proof}
Fix an arm $a$, an objective $j$, and a count $n \in \{1,\dots,T\}$. Let
\[
\bar Y_{a,n}^{(j)} := \frac{1}{n}\sum_{m=1}^n Y_{a,m}^{(j)}.
\]
Because the latent samples
\[
Y_{a,1}^{(j)},Y_{a,2}^{(j)},\dots
\]
are i.i.d.\ with mean $\mu_a^{(j)}$, the average $\bar Y_{a,n}^{(j)}$ is the mean of $n$ i.i.d.\ $[0,1]$-valued random variables. By Hoeffding's inequality,
\[
\Pr\!(
\left|\bar Y_{a,n}^{(j)}-\mu_a^{(j)}\right| > \sqrt{\frac{2\log T}{n}}
)
\le
2\exp(-4\log T)
=
2T^{-4}.
\]
Define
\[
E_{a,j,n}
:=
\left\{
\left|\bar Y_{a,n}^{(j)}-\mu_a^{(j)}\right|
>
\sqrt{\frac{2\log T}{n}}
\right\}.
\]
Then $\Pr(E_{a,j,n}) \le 2T^{-4}$.

Now fix any round $K<t\le T$. By the latent-sequence construction, the empirical mean of arm $a$ and objective $j$ after $t-1$ rounds is the average of the first $N_a(t-1)$ latent samples from that arm, namely
\[
\widehat{\mu}_a^{(j)}(t-1)=\bar Y_{a,N_a(t-1)}^{(j)}.
\]
Hence if $\Omega_T$ fails, then for some $(t,a,j)$ with $K<t\le T$,
\[
\left|\widehat{\mu}_a^{(j)}(t-1)-\mu_a^{(j)}\right|
>
\sqrt{\frac{2\log T}{N_a(t-1)}}.
\]
Setting $n:=N_a(t-1)$, the event $E_{a,j,n}$ occurs. Therefore
\[
\Omega_T^c \subseteq \bigcup_{a \in [K]} \bigcup_{j \in [d]} \bigcup_{n=1}^T E_{a,j,n}.
\]
Applying the union bound yields
\[
\Pr(\Omega_T^c) \le 2KdT\cdot T^{-4} = 2Kd\,T^{-3}.
\]
\end{proof}

\subsubsection{Proof of Lemma~\ref{lem:certify}}

\begin{proof}
Since $b_j(t)$ and $c_j(t)$ are the two largest-UCB arms on objective $j$, every arm $a \neq b_j(t)$ satisfies
\[
U_a^{(j)}(t) \le U_{c_j(t)}^{(j)}(t) < L_{b_j(t)}^{(j)}(t) \le \mu_{b_j(t)}^{(j)}.
\]
Also, on $\Omega_T$,
\[
\mu_a^{(j)} \le U_a^{(j)}(t).
\]
Hence
\[
\mu_a^{(j)} < \mu_{b_j(t)}^{(j)}
\qquad \forall a \neq b_j(t),
\]
so $b_j(t)$ is the unique objective maximizer on objective $j$. No other arm can dominate it, because every other arm is strictly worse on coordinate $j$. Therefore $b_j(t)\in A^\star$, and hence $\Delta_{b_j(t)}^{\mathrm P}=0$.
\end{proof}

\subsubsection{Proof of Lemma~\ref{lem:local_charge}}

\begin{proof}
Let $j_t$ be the selected objective and write
\[
b_t := b_{j_t}(t),
\qquad
c_t := c_{j_t}(t).
\]
For any arm $a$ and objective $j$, define
\[
m_j := \max_{u\in[K]}\mu_u^{(j)},
\qquad
\delta_j(a) := m_j - \mu_a^{(j)} \ge 0.
\]
We first claim that for every arm $a$ and objective $j$,
\[
\Delta_a^{\mathrm P} \le \delta_j(a).
\]
Fix any $\eta > \delta_j(a)$. Then
\[
\mu_a^{(j)}+\eta > m_j \ge \mu_b^{(j)}
\qquad \forall b \in [K],
\]
so the lifted vector $\mu_a+\eta\ones$ is strictly largest in coordinate $j$ and therefore cannot be dominated by any arm. By the definition of the Pareto suboptimality gap, $\Delta_a^{\mathrm P}\le \eta$ for every $\eta>\delta_j(a)$, and thus $\Delta_a^{\mathrm P}\le \delta_j(a)$.

Applying this with $j=j_t$, it suffices to upper-bound
\[
\delta_{j_t}(A_t)=m_{j_t}-\mu_{A_t}^{(j_t)}.
\]
Because round $t$ is non-certifying, $A_t \in \{b_t,c_t\}$.

\paragraph{Case 1: $A_t=b_t$.}
If $\mu_{b_t}^{(j_t)}=m_{j_t}$, then $A_t$ maximizes objective $j_t$ and $\Delta_{A_t}^{\mathrm P}=0$. Otherwise, choose any $a^+\in\argmax_{a\in[K]}\mu_a^{(j_t)}$. Since $b_t$ has the largest UCB on objective $j_t$,
\[
U_{b_t}^{(j_t)}(t) \ge U_{a^+}^{(j_t)}(t) \ge m_{j_t},
\]
while on $\Omega_T$,
\[
U_{b_t}^{(j_t)}(t) \le \mu_{b_t}^{(j_t)} + 2\beta_{b_t}(t).
\]
Hence
\[
m_{j_t} - \mu_{b_t}^{(j_t)} \le 2\beta_{b_t}(t),
\]
so $\delta_{j_t}(A_t)\le 2\beta_{A_t}(t)$ and therefore
\[
\Delta_{A_t}^{\mathrm P}\le 2\beta_{A_t}(t).
\]

\paragraph{Case 2: $A_t=c_t$.}
If $\mu_{b_t}^{(j_t)}=m_{j_t}$, then because round $t$ is non-certifying,
\[
L_{b_t}^{(j_t)}(t) \le U_{c_t}^{(j_t)}(t).
\]
On $\Omega_T$,
\[
L_{b_t}^{(j_t)}(t) \ge m_{j_t} - 2\beta_{b_t}(t),
\qquad
U_{c_t}^{(j_t)}(t) \le \mu_{c_t}^{(j_t)} + 2\beta_{c_t}(t),
\]
which gives
\[
m_{j_t} - \mu_{c_t}^{(j_t)}
\le
2\beta_{b_t}(t)+2\beta_{c_t}(t).
\]
Since $A_t=c_t$ is the more uncertain endpoint,
\[
\beta_{b_t}(t)\le \beta_{c_t}(t)=\beta_{A_t}(t),
\]
so $\delta_{j_t}(A_t) \le 4\beta_{A_t}(t)$ and hence
\[
\Delta_{A_t}^{\mathrm P} \le 4\beta_{A_t}(t).
\]

If instead $\mu_{b_t}^{(j_t)}<m_{j_t}$, then some objective maximizer $a^+\in\argmax_{a\in[K]}\mu_a^{(j_t)}$ belongs to the candidate set for $c_t$, so
\[
U_{c_t}^{(j_t)}(t) \ge U_{a^+}^{(j_t)}(t) \ge m_{j_t}.
\]
On $\Omega_T$,
\[
U_{c_t}^{(j_t)}(t) \le \mu_{c_t}^{(j_t)} + 2\beta_{c_t}(t),
\]
which yields
\[
m_{j_t}-\mu_{c_t}^{(j_t)} \le 2\beta_{c_t}(t)=2\beta_{A_t}(t).
\]
Thus $\Delta_{A_t}^{\mathrm P}\le 2\beta_{A_t}(t)\le 4\beta_{A_t}(t)$.
\end{proof}

\subsubsection{Proof of Lemma~\ref{lem:width_floor}}

\begin{proof}
Since $g^\dagger>0$, objective $j^\dagger$ has a separated winner, denoted by $a_{j^\dagger}^\star$.
Let
\[
b^\dagger := b_{j^\dagger}(t),
\qquad
c^\dagger := c_{j^\dagger}(t).
\]

\paragraph{Case 1: $b^\dagger = a_{j^\dagger}^\star$.}
Because round $t$ is non-certifying, objective $j^\dagger$ does not certify, so
\[
L_{b^\dagger}^{(j^\dagger)}(t) \le U_{c^\dagger}^{(j^\dagger)}(t).
\]
On $\Omega_T$,
\[
L_{b^\dagger}^{(j^\dagger)}(t)
\ge
\mu_{a_{j^\dagger}^\star}^{(j^\dagger)} - 2\beta_{b^\dagger}(t).
\]
Since \(c^\dagger \neq a_{j^\dagger}^\star\) and
\(g^\dagger=g_{j^\dagger}\), the definition of \(g_{j^\dagger}\) gives
\[
\mu_{c^\dagger}^{(j^\dagger)}
\le
\mu_{a_{j^\dagger}^\star}^{(j^\dagger)} - g^\dagger.
\]
Therefore, on $\Omega_T$,
\[
U_{c^\dagger}^{(j^\dagger)}(t)
\le
\mu_{a_{j^\dagger}^\star}^{(j^\dagger)} - g^\dagger + 2\beta_{c^\dagger}(t).
\]
Combining the displays gives
\[
2\beta_{b^\dagger}(t)+2\beta_{c^\dagger}(t) \ge g^\dagger,
\]
and hence
\[
\beta_{b^\dagger}(t)+\beta_{c^\dagger}(t) \ge \frac{g^\dagger}{2}.
\]

\paragraph{Case 2: $b^\dagger \neq a_{j^\dagger}^\star$.}
By definition of $b^\dagger$,
\[
U_{b^\dagger}^{(j^\dagger)}(t)
\ge
U_{a_{j^\dagger}^\star}^{(j^\dagger)}(t)
\ge
\mu_{a_{j^\dagger}^\star}^{(j^\dagger)}.
\]
On $\Omega_T$,
\[
U_{b^\dagger}^{(j^\dagger)}(t)
\le
\mu_{b^\dagger}^{(j^\dagger)} + 2\beta_{b^\dagger}(t)
\le
\mu_{a_{j^\dagger}^\star}^{(j^\dagger)} - g^\dagger + 2\beta_{b^\dagger}(t),
\]
which implies
\[
\beta_{b^\dagger}(t)\ge \frac{g^\dagger}{2}.
\]
Thus
\[
\beta_{b^\dagger}(t)+\beta_{c^\dagger}(t)\ge \frac{g^\dagger}{2}.
\]

In both cases,
\[
\beta_{b^\dagger}(t)+\beta_{c^\dagger}(t)\ge \frac{g^\dagger}{2}.
\]
Since round $t$ is non-certifying, the algorithm picks
\[
j_t \in \arg\max_{j\in[d]} W_j(t),
\]
so
\[
W_{j_t}(t)\ge W_{j^\dagger}(t).
\]
Recalling that $W_j(t)=\beta_{b_j(t)}(t)+\beta_{c_j(t)}(t)$, we obtain
\[
\beta_{b_t}(t)+\beta_{c_t}(t)\ge \frac{g^\dagger}{2}.
\]
Finally,
\[
\beta_{A_t}(t)=\max\{\beta_{b_t}(t),\beta_{c_t}(t)\}
\ge
\frac{\beta_{b_t}(t)+\beta_{c_t}(t)}{2}
\ge
\frac{g^\dagger}{4}.
\]
\end{proof}

\subsubsection{Proof of Theorem~\ref{thm:regret}}

\begin{proof}
Let $\tau_{\mathrm{cert}}$ be the first round after the warm start at which some objective certifies, with the convention $\tau_{\mathrm{cert}}:=T+1$ if no certification occurs by time $T$.

\paragraph{Good-event decomposition.}
The warm start contributes at most
\[
K\Delta_{\max}^{\mathrm P}.
\]
On $\Omega_T$, Lemma~\ref{lem:certify} shows that the first certified leader has zero Pareto gap. Since the algorithm stores that leader and exploits it forever,
\[
\Delta_{A_t}^{\mathrm P}=0
\qquad\text{for every } t \ge \tau_{\mathrm{cert}}
\quad\text{on } \Omega_T.
\]
Thus only rounds $K<t<\tau_{\mathrm{cert}}$ contribute further regret on $\Omega_T$.

\paragraph{Length of the pre-certification phase.}
Fix any round $K<t<\tau_{\mathrm{cert}}$ with $A_t=a$. Since round $t$ is non-certifying, Lemma~\ref{lem:width_floor} gives
\[
\beta_{A_t}(t) \ge \frac{g^\dagger}{4}.
\]
Hence
\[
\sqrt{\frac{2\log T}{N_a(t-1)}} \ge \frac{g^\dagger}{4},
\]
which implies
\[
N_a(t-1) \le \frac{32\log T}{(g^\dagger)^2}.
\]
Now let $m_a$ be the number of times arm $a$ is pulled after the warm start and before certification. The successive pre-pull counts of those pulls are exactly
\[
1,2,\dots,m_a.
\]
Since each must be at most $32\log T/(g^\dagger)^2$, we obtain the direct bound
\[
m_a \le \frac{32\log T}{(g^\dagger)^2}.
\]

\paragraph{Regret accumulated before certification.}
For each arm $a$, let
\[
t_{a,1}<t_{a,2}<\cdots<t_{a,m_a}
\]
be the rounds at which arm $a$ is pulled after the warm start and before certification. Then
\[
\sum_{t=K+1}^{\tau_{\mathrm{cert}}-1}\Delta_{A_t}^{\mathrm P}
=
\sum_{a=1}^K \sum_{s=1}^{m_a}\Delta_{A_{t_{a,s}}}^{\mathrm P}.
\]
Since each $t_{a,s}$ is non-certifying, Lemma~\ref{lem:local_charge} gives
\[
\Delta_{A_{t_{a,s}}}^{\mathrm P} \le 4\beta_{A_{t_{a,s}}}(t_{a,s})=4\beta_a(t_{a,s}).
\]
Because $t_{a,s}$ is the $s$th post-warm-start pull of arm $a$,
\[
N_a(t_{a,s}-1)=s,
\qquad
\beta_a(t_{a,s})=\sqrt{\frac{2\log T}{s}}.
\]
Therefore
\[
\sum_{t=K+1}^{\tau_{\mathrm{cert}}-1}\Delta_{A_t}^{\mathrm P}
\le
4\sum_{a=1}^K \sum_{s=1}^{m_a}\sqrt{\frac{2\log T}{s}}.
\]
Using $\sum_{s=1}^{m_a}s^{-1/2}\le 2\sqrt{m_a}$ and the bound on $m_a$ gives
\begin{align*}
\sum_{a=1}^K \sum_{s=1}^{m_a}\sqrt{\frac{2\log T}{s}}
&\le
\sum_{a=1}^K \sqrt{2\log T}\sum_{s=1}^{m_a}s^{-1/2} \\
&\le
2\sqrt{2\log T}\sum_{a=1}^K \sqrt{m_a} \\
&\le
2K\sqrt{2\cdot \frac{32\log T}{(g^\dagger)^2}\cdot \log T}
\le
16\,\frac{K\log T}{g^\dagger}.
\end{align*}
Hence
\[
\sum_{t=K+1}^{\tau_{\mathrm{cert}}-1}\Delta_{A_t}^{\mathrm P}
\le
64\,\frac{K\log T}{g^\dagger}.
\]
Combining this estimate with zero regret after certification yields
\[
R_T^{\mathrm P}\mathbf{1}\{\Omega_T\}
\le
K\Delta_{\max}^{\mathrm P}
+
64\,\frac{K\log T}{g^\dagger}.
\]

\paragraph{Bad-event contribution.}
Trivially,
\[
R_T^{\mathrm P}\le T\Delta_{\max}^{\mathrm P}.
\]
Therefore Proposition~\ref{prop:confidence} yields
\[
\E\!\left[R_T^{\mathrm P}\mathbf{1}\{\Omega_T^c\}\right]
\le
T\Delta_{\max}^{\mathrm P}\Pr(\Omega_T^c)
\le
2Kd\,\Delta_{\max}^{\mathrm P}T^{-2}.
\]
Adding the bad-event term proves the theorem.
\end{proof}

\subsubsection{Comparison with Pareto UCB1}
\label{app:comparison_pucb}

Theorem~1 of \citep{drugan2013designing} states that the Pareto UCB1 policy satisfies
\[
\E\!\left[R_T^{\mathrm P}\right]
\le
\sum_{i\notin A^\star}
\frac{8\,\log\!(T(d|A^\star|)^{1/4})}{\Delta_i^{\mathrm P}}
\;+\;
(1+\frac{\pi^2}{3})\sum_{i\notin A^\star}\Delta_i^{\mathrm P}.
\]
This comparison is most transparent in the generic non-boundary case where every arm outside the Pareto set has positive Pareto gap. If boundary dominated arms with zero Pareto gap exist, they are regret-neutral under the criterion and the same comparison should be read over the positive-gap arms. If $A^\star=[K]$, then the leading term is the empty sum and there is nothing further to compare. We therefore focus on the nontrivial generic case $A^\star\neq[K]$ and define
\[
\Delta_{\min}^{\mathrm P} := \min_{i\notin A^\star}\Delta_i^{\mathrm P}.
\]
To compare leading terms, factor out the common logarithmic term:
\[
\sum_{i\notin A^\star}
\frac{8\,\log\!(T(d|A^\star|)^{1/4})}{\Delta_i^{\mathrm P}}
=
8\,\log\!(T(d|A^\star|)^{1/4})
\sum_{i\notin A^\star}\frac{1}{\Delta_i^{\mathrm P}}.
\]
Since every dominated arm satisfies $\Delta_i^{\mathrm P}\ge \Delta_{\min}^{\mathrm P}$, we have
\[
\sum_{i\notin A^\star}\frac{1}{\Delta_i^{\mathrm P}}
\le
\frac{K-|A^\star|}{\Delta_{\min}^{\mathrm P}}.
\]
Therefore the leading term is at most
\[
\frac{8(K-|A^\star|)\log\!(T(d|A^\star|)^{1/4})}{\Delta_{\min}^{\mathrm P}}.
\]
Writing this on the same $K\log T/g^\dagger$ scale as Theorem~\ref{thm:regret} gives
\[
\frac{8(K-|A^\star|)\log\!(T(d|A^\star|)^{1/4})}{\Delta_{\min}^{\mathrm P}}
=
(
8\cdot \frac{K-|A^\star|}{K}
\cdot
\frac{g^\dagger}{\Delta_{\min}^{\mathrm P}}
\cdot
\frac{\log\!(T(d|A^\star|)^{1/4})}{\log T}
)
\frac{K\log T}{g^\dagger}.
\]
From the original leading term, the exact normalized coefficient induced by \citep{drugan2013designing} is
\[
C_{\mathrm{PUCB}}
:=
8\cdot \frac{g^\dagger}{K}
(\sum_{i\notin A^\star}\frac{1}{\Delta_i^{\mathrm P}})
\frac{\log\!(T(d|A^\star|)^{1/4})}{\log T}.
\]
Our leading coefficient $64$ is smaller exactly when
\[
64 < C_{\mathrm{PUCB}},
\]
that is, when
\[
\frac{g^\dagger}{K}\sum_{i\notin A^\star}\frac{1}{\Delta_i^{\mathrm P}}
>
8\cdot \frac{\log T}{\log\!(T(d|A^\star|)^{1/4})}.
\]
This is the exact pointwise comparison between the two leading terms.

The $\Delta_{\min}^{\mathrm P}$ rewriting used in Remark~\ref{rem:compare_pucb} is a coarser upper envelope. Indeed,
\[
C_{\mathrm{PUCB}}
\le
\overline{C}_{\mathrm{PUCB}}
:=
8\cdot \frac{K-|A^\star|}{K}
\cdot
\frac{g^\dagger}{\Delta_{\min}^{\mathrm P}}
\cdot
\frac{\log\!(T(d|A^\star|)^{1/4})}{\log T}.
\]
Using
\[
\frac{\log T}{\log\!(T(d|A^\star|)^{1/4})}
=
(
1+\frac{\log(d|A^\star|)}{4\log T}
)^{-1},
\]
the factor multiplying $g^\dagger/\Delta_{\min}^{\mathrm P}$ is asymptotically close to $8(K-|A^\star|)/K$. This upper envelope highlights the setting in which the two bounds separate most clearly: when $\Delta_{\min}^{\mathrm P}\ll g^\dagger$, or more generally when $\sum_{i\notin A^\star}1/\Delta_i^{\mathrm P}$ is large, the earlier leading coefficient can be much larger than $64$. If instead the dominated-arm Pareto gaps are all comparable to $g^\dagger$ and only a few dominated arms remain, the earlier coefficient is also constant-order and can be numerically smaller.

\subsubsection{Proof of Proposition~\ref{prop:lower_bound}}

\begin{proof}
We first compute the Pareto geometry of the subclass and verify that the theorem parameter satisfies $g^\dagger=\Delta_{\mathrm{sc}}$. We then show that Pareto regret reduces exactly to the ordinary single-objective regret of the reduced Bernoulli bandit instance. Finally, because each reward vector is one Bernoulli sample duplicated across coordinates, the problem is equivalent to that reduced single-objective Bernoulli bandit, so Theorem~2 of \citep{lai} applies.

Set
\[
\theta_1=\frac12+\Delta_{\mathrm{sc}},
\qquad
\theta_a=\frac12
\quad\text{for } a=2,\ldots,K.
\]
Then every arm has mean vector
\[
\mu_a = \theta_a \ones.
\]
Hence
\[
\mu_1^{(j)}=\frac12+\Delta_{\mathrm{sc}}>\frac12=\mu_a^{(j)}
\qquad\forall a\ge 2,\ \forall j\in[d],
\]
so arm $1$ dominates every other arm and therefore
\[
A^\star=\{1\}.
\]
Because arm $1$ is the unique maximizer on every objective,
\[
g_j
=
\mu_1^{(j)}-\max_{a\neq 1}\mu_a^{(j)}
=
\Delta_{\mathrm{sc}}
\qquad\forall j\in[d].
\]
Thus $g^\dagger=\Delta_{\mathrm{sc}}$.

For arm $1$, $\Delta_1^{\mathrm P}=0$. Fix any arm $a\ge 2$. Since
\[
\mu_a+\epsilon \ones=(\tfrac12+\epsilon)\ones,
\]
and every arm $b\ge 2$ has mean vector $(1/2)\ones$, it is enough to compare only with $\mu_1$. The lifted vector is dominated by $\mu_1=(\tfrac12+\Delta_{\mathrm{sc}})\ones$ whenever $\epsilon<\Delta_{\mathrm{sc}}$, and it is no longer dominated once $\epsilon\ge \Delta_{\mathrm{sc}}$. Therefore
\[
\Delta_a^{\mathrm P}=\Delta_{\mathrm{sc}}
\qquad\text{for every } a=2,\ldots,K.
\]
It follows that
\[
R_T^{\mathrm P}
=
\sum_{a=2}^K \Delta_a^{\mathrm P} N_a(T)
=
\Delta_{\mathrm{sc}}\sum_{a=2}^K N_a(T).
\]

Because each reward vector is obtained by duplicating one scalar Bernoulli sample across all coordinates, observing
\[
Y_{a,n}=(Z_{a,n},\ldots,Z_{a,n})
\]
is exactly the same as observing the scalar sample $Z_{a,n}$. Thus one pull in the duplicated-coordinate instance provides exactly the same information as one pull in the reduced single-objective $K$-armed Bernoulli bandit with means
\[
\theta_1=\frac12+\Delta_{\mathrm{sc}},
\qquad
\theta_a=\frac12
\quad\text{for } a=2,\ldots,K.
\]

Let $N_a(T)$ denote the pull count of arm $a$ up to time $T$. By assumption, the policy induced by the duplicated-coordinate reduction is uniformly good for the reduced Bernoulli bandit. We may therefore apply Theorem~2 of \citep{lai}, which lower-bounds the expected number of samples drawn from an inferior population under such a rule. Fix a suboptimal arm $a\in\{2,\ldots,K\}$ and match it to population $j=a$ in Lai--Robbins. In our reduced Bernoulli bandit,
\[
p(\theta_a)=\frac12,
\qquad
p(\theta^\star)=\frac12+\Delta_{\mathrm{sc}}.
\]
The quantity $I(\theta_a,\theta^\star)$ is the least information cost of changing the inferior population into one whose mean exceeds $p(\theta^\star)$. For Bernoulli distributions this is the Bernoulli Kullback--Leibler divergence
\[
\mathrm{kl}(p,q)
:=
p\log\frac{p}{q}+(1-p)\log\frac{1-p}{1-q}.
\]
Since the inferior arm has Bernoulli mean $1/2$, the alternative means making it better than the current best arm are exactly those with
\[
q>\frac12+\Delta_{\mathrm{sc}}.
\]
Therefore,
\[
I(\theta_a,\theta^\star)
=
\inf_{q>1/2+\Delta_{\mathrm{sc}}}\mathrm{kl}\!(\frac12,q).
\]
For fixed $p=1/2$, the map $q\mapsto \mathrm{kl}(1/2,q)$ is increasing on $(1/2,1)$, so
\[
I(\theta_a,\theta^\star)
=
\mathrm{kl}\!(\frac12,\frac12+\Delta_{\mathrm{sc}}).
\]
Hence Theorem~2 of \citep{lai} gives, for each suboptimal arm $a=2,\ldots,K$,
\[
\liminf_{T\to\infty}\frac{\E[N_a(T)]}{\log T}
\ge
\frac{1}{\mathrm{kl}(1/2,\ 1/2+\Delta_{\mathrm{sc}})}.
\]
Summing over $a=2,\ldots,K$ and using
\[
R_T^{\mathrm P}
=
\Delta_{\mathrm{sc}}\sum_{a=2}^K N_a(T)
\]
yields
\[
\liminf_{T\to\infty}\frac{\E[R_T^{\mathrm P}]}{\log T}
\ge
(K-1)\frac{\Delta_{\mathrm{sc}}}{\mathrm{kl}(1/2,\ 1/2+\Delta_{\mathrm{sc}})}.
\]

Next,
\[
\mathrm{kl}\!(\frac12,\frac12+\Delta_{\mathrm{sc}})
=
-\frac12\log(1-4\Delta_{\mathrm{sc}}^2).
\]
Using
\[
-\log(1-x)\le \frac{x}{1-x}
\qquad\text{for } 0\le x<1,
\]
with $x=4\Delta_{\mathrm{sc}}^2$, we obtain
\[
\mathrm{kl}\!(\frac12,\frac12+\Delta_{\mathrm{sc}})
\le
\frac{2\Delta_{\mathrm{sc}}^2}{1-4\Delta_{\mathrm{sc}}^2}.
\]
Because $\Delta_{\mathrm{sc}}\le 1/4$, we have $1-4\Delta_{\mathrm{sc}}^2\ge 3/4$, and therefore
\[
\mathrm{kl}\!(\frac12,\frac12+\Delta_{\mathrm{sc}})
\le
\frac{8}{3}\Delta_{\mathrm{sc}}^2.
\]
Substituting this bound yields
\[
\liminf_{T\to\infty}\frac{\E[R_T^{\mathrm P}]}{\log T}
\ge
\frac{3}{8}\frac{K-1}{\Delta_{\mathrm{sc}}}.
\]
Since $g^\dagger=\Delta_{\mathrm{sc}}$, the claimed $\Omega(K\log T/g^\dagger)$ lower bound follows.

\end{proof}

\end{document}